\documentclass[11pt]{article}

\usepackage[preprint]{acl}

\usepackage{times}
\usepackage{latexsym}
\usepackage[T1]{fontenc}
\usepackage[utf8]{inputenc}
\usepackage{microtype}
\usepackage{inconsolata}
\usepackage{url}
\usepackage{subcaption}
\usepackage{booktabs}
\usepackage{graphicx}
\usepackage{multirow}
\usepackage{array}
\usepackage{makecell}
\usepackage[table]{xcolor}

\usepackage{cuted}
\usepackage{amsmath}
\usepackage{mathtools}
\usepackage{amsfonts}
\usepackage{amssymb}
\usepackage{amsthm}

\usepackage{algorithm}
\usepackage{algpseudocode}

\usepackage{tcolorbox}
\tcbuselibrary{skins,breakable}

\usepackage{placeins}
\usepackage{longtable}

\usepackage{hyperref}
\usepackage[capitalize,noabbrev]{cleveref}

\newtheorem{lemma}{Lemma}
\newtheorem{proposition}{Proposition}

\newcommand{\KL}{D_{\mathrm{KL}}}

\newcommand{\hist}{h}
\newcommand{\canon}{c}
\newcommand{\task}{q}

\newcommand{\yseq}{y}

\newcommand{\base}{\pi_{0}}
\newcommand{\student}{\pi_{\theta}}

\newcommand{\ccopd}{\textsc{CCOPD}}
\newcommand{\full}{\textsc{Full}}
\newcommand{\concat}{\textsc{Concat}}
\newcommand{\raw}{\textsc{Raw-Sharded}}
\newcommand{\waittraj}{\textsc{Wait}}
\newcommand{\useronly}{\textsc{User-Only}}

\title{Same Evidence, Different Answers: Canonical-Context On-Policy Distillation for Multi-Turn Language Models}

 \author{ \begin{tabular}{@{}c@{}} \large Zizhuo Lin$^{1,\dagger}$ \quad Quanling Liu$^{1}$ \quad Jinsheng Quan$^{1}$ \quad Chao Zhang$^{1}$ \quad Yifan Zhu$^{1}$ \\ \large Xing Shi$^{1}$ \quad Jingtao Xu$^{1}$ \quad Zhihui Li$^{2}$ \quad Yawei Luo$^{1}$ \\ \normalsize $^{1}$Zhejiang University \quad $^{2}$University of Science and Technology of China \\ \small $^{\dagger}$Email: \texttt{zizhuolin@zju.edu.cn} \end{tabular} }

\begin{document}
\maketitle

\begin{abstract}
Large language models (LLMs) often solve a task when all instructions are given in a single prompt, but fail when the same information is revealed gradually across turns. When a clean FULL prompt and a RAW-SHARDED conversation contain the same complete user evidence, the model should still arrive at the same answer. We argue that a key reason for this gap is self-anchored drift: responses produced under partial information introduce unsupported assumptions, and those assumptions later distort the final answer. To reduce this effect, we propose Canonical-Context On-Policy Distillation (CCOPD). During training, the same model is used in two roles: a frozen teacher conditioned on the clean FULL prompt and a trainable student that receives the same evidence incrementally through a multi-turn conversation; CCOPD aligns the student’s behavior on its own trajectories with the teacher’s canonical full-context behavior. Trained only on math problem conversations, CCOPD yields a 32\% relative improvement in the instance-weighted aggregate score in RAW-SHARDED performance over the original model across math and five zero-shot out-of-domain task families, while largely preserving full-context performance. Further analyses suggest that CCOPD strengthens grounding in user evidence and reduces sensitivity to contamination from earlier assistant turns.
\end{abstract}

\section{Introduction}
% Users often reveal task constraints over several turns rather than in one fully formed prompt. When complete task-relevant user evidence is revealed incrementally over the course of a multi-turn interaction, a basic reliability requirement is that, once the conversation is complete, the model should behave as if the same evidence had been presented in a clean one-shot prompt. We call this requirement \emph{canonical-context consistency}.
% Recent multi-turn evaluations show that canonical-context consistency is fragile: models that succeed when a task is presented as a clean \full{} prompt  can still fail when the same information is disclosed incrementally through a \raw{} conversation \citep{laban2025lost,he2024multiif,li2025structflowbench}. Part 1 of Figure~\ref{fig:teaser} provides a concrete example of the three task-equivalent presentation modes.  

% \begin{f}
%   \centering
%   \includegraphics[width=\linewidth]{figures/teaser.pdf}
% \captionof{figure}{Overview of \ccopd{}. 
% \textbf{Part 1} contrasts task-equivalent \full{}, \textsc{Concat}, and \raw{} presentations. 
% \textbf{Part 2} distills a \raw{}-conditioned student from a frozen \full{}-conditioned teacher via token-level reverse KL on final-answer prefixes. 
% \textbf{Part 3} illustrates reduced self-anchored drift and improved canonical-context consistency.}
%   \label{fig:teaser}
% \end{strip}

\begin{figure}[t]
\centering
\includegraphics[width=\linewidth]{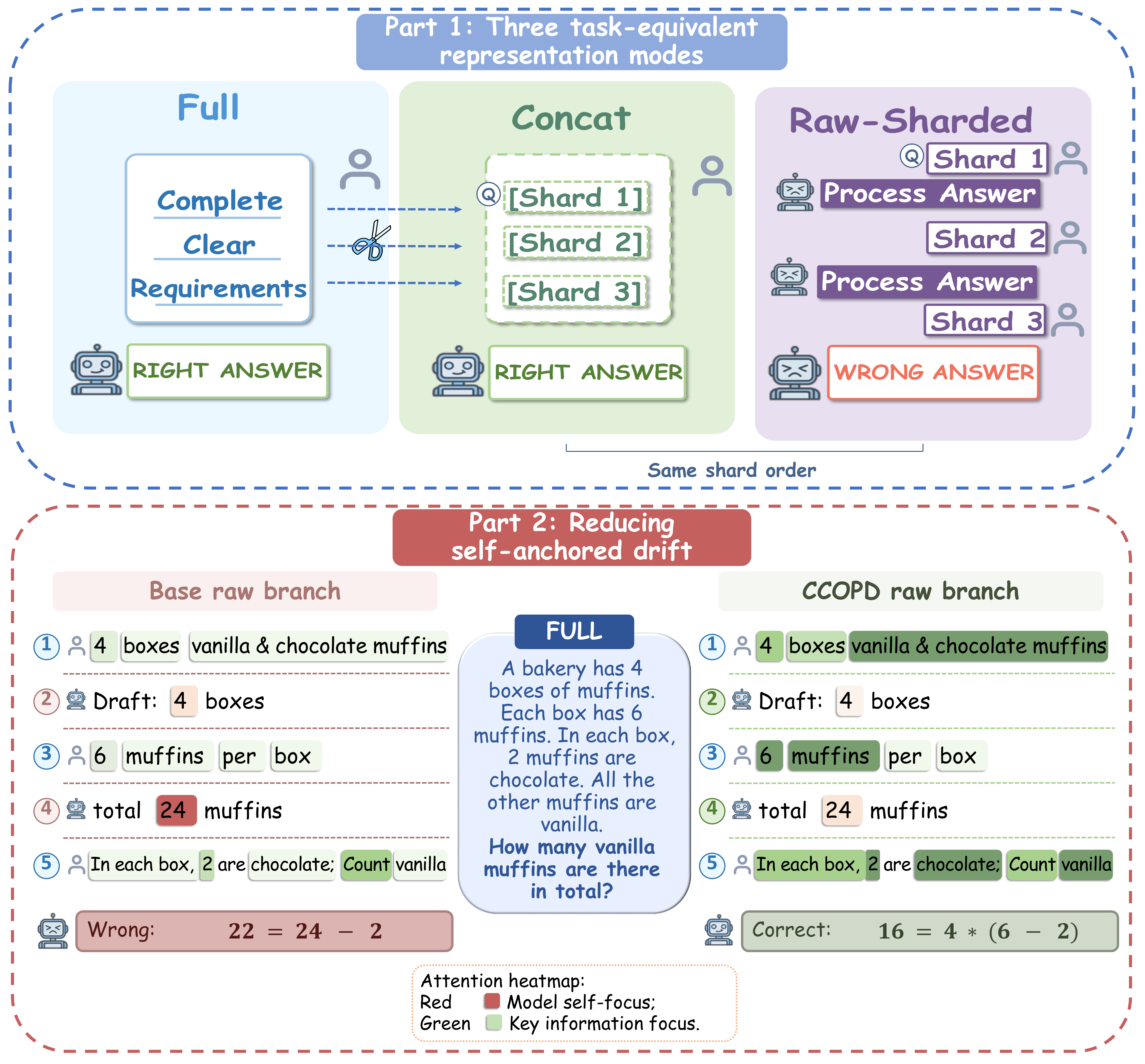}
\caption{%
\textbf{Part 1}: Task-equivalent \full{}, \textsc{Concat}, and \raw{} presentations. We use \textsc{Concat} to compare and control for the effects of \textsc{Raw-Shard} segmentation and ordering.
\textbf{Part 2}: Reduced self-anchored drift and improved canonical-context consistency.}
\label{fig:teaser}
\end{figure}

% \begin{figure*}[t]
%     \centering
%     \includegraphics[width=0.49\linewidth]{figures/teaser-part1.pdf}
%     \hfill
%     \includegraphics[width=0.49\linewidth]{figures/teaser-part2.pdf}
%     \caption{%
%     \textbf{Part 1}: Task-equivalent \full{}, \textsc{Concat}, and \raw{} presentations.
%     \textbf{Part 2}: Reduced self-anchored drift and improved canonical-context consistency.}
%     \label{fig:teaser}
% \end{figure*}

Users rarely specify a task in one fully formed prompt: they often reveal constraints over several turns. %When complete task-relevant user evidence is revealed incrementally over the course of a multi-turn interaction, a basic reliability requirement is that, once the conversation is complete, the model should behave as if the same evidence had been presented in a clean one-shot prompt. 
In multi-turn interactions where task-related user information is provided gradually, a fundamental reliability criterion requires the model to yield the same outcomes as it would when all such information is given in a single prompt after the conversation ends. We call this requirement \emph{canonical-context consistency}. Recent multi-turn evaluations show that this consistency is fragile: models that succeed when a task is presented as a clean \full{} prompt can still fail when the same information is disclosed incrementally through a \raw{} conversation \citep{laban2025lost,he2024multiif,li2025structflowbench}. Part~1 of Figure~\ref{fig:teaser} provides a concrete example of the three task-equivalent presentation modes.

% The core difficulty is that a \raw{} history is not merely a longer prompt. It also contains the model's own earlier replies, produced before the task is fully specified. Those replies can introduce tentative answers, unsupported assumptions, or partial reasoning into the context. By the final turn, the model may rely on this self-generated text instead of fully re-grounding its answer in the user-provided shards. We call this failure mode \emph{self-anchored drift}.

The core difficulty is that a \raw{} history is not merely a longer
prompt. It also contains the model's own earlier replies, produced before
the task evidence is complete. Those replies can introduce tentative
answers, unsupported assumptions, or partial reasoning into the context.
By the final turn, the model may rely on this self-generated text instead
of fully re-grounding its answer in the completed user-provided shards.
We call this failure mode \emph{self-anchored drift}.

An alternative is to repair the trajectory at inference time. Such methods can reflect on, revise, reset, or consolidate intermediate reasoning before the model commits to a final answer \citep{shinn2023reflexion,madaan2023selfrefine,mohammad-khalid-etal-2025-ergo}. Although they are often useful in deployment, they rely on an additional external control loop and thus modify the inference process itself. Another line of work addresses the problem through clarification and abstention: when a request is ambiguous or the information needed for answering is still missing, the model must decide whether to ask a clarifying question, defer its response, or abstain \citep{zhang2025modeling,zhang2025clarify,testoni2024asking,wu2025collabllm,zhang2024rtuning,cheng2024can}. However, recent evidence suggests that current off-the-shelf LLMs still struggle to reliably identify ambiguous requests and ask appropriate clarifying questions in practice, while still introducing contamination from the model's own prior responses \citep{zhang2024clamber,li2025questbench,luo2025clarifymt}.

We argue that \textbf{canonical-context consistency should be internalized as a core model capability}: once the final turn provides sufficient user evidence, the model's  answer must be anchored to that evidence rather than deviating toward assumptions introduced by itself. To this end, we propose \textbf{Canonical-Context On-Policy Distillation} (\ccopd{}), a self-distillation objective that enforces canonical-context consistency under contaminated multi-turn histories. The same base model is adopted in two roles: a frozen teacher reads the clean canonical \full{} task, while the trainable student receives the task only through the gradually revealed \raw{} interaction. Once all task-relevant evidence has been revealed through the interaction, the \full{}-conditioned teacher provides token-level supervision on the student’s own final-answer prefixes. This same-prefix, on-policy supervision re-anchors the \raw{}-conditioned student, encouraging its answer distribution to move away from self-generated contamination and toward the canonical distribution grounded in the completed user evidence.  In this sense, \ccopd{} does not rely on a stronger external teacher.
It targets cases where the original model can solve the task under the task-equivalent clean \full{} prompt, but may fail under the  \raw{} history. 
\ccopd{} trains the model to preserve that \full{}-context capability in the self-contaminated \raw{} setting.

Under the task-equivalent sharding setup introduced by \citet{laban2025lost}, \ccopd{} is trained only on sharded GSM8K-style math conversations.
Strikingly, this math-only post-training signal transfers well beyond math.
It improves \raw{} accuracy on math and yields zero-shot gains across five non-math settings. Overall, \ccopd{} yields a 32\% relative improvement in the instance-weighted aggregate score over the unmodified base model on \raw{} performance, while largely preserving \full{} and \concat{} performance.

 Our contributions are:

\begin{itemize}
\item We formalize canonical-context consistency for task-equivalent
histories and identify self-anchored drift as a concrete source of its
failure. We introduce \ccopd{}, an on-policy same-prefix distillation
objective that aligns \raw{} final-answer behavior with a frozen
same-backbone \full{} teacher, without adding a stronger external teacher
or an inference-time repair loop.

\item We show that a math-only \ccopd{} signal transfers beyond math:
the same adapter improves \raw{} performance across all six task
families, including five non-math zero-shot settings, while largely
preserving \full{} and \concat{} performance. These gains exceed the
data-matched post-training baselines in our evaluation.

\item We support the proposed mechanism with a suite of ablations and diagnostics. These analyses show that the gains are tied to canonical \full{}-view alignment and reverse-KL same-prefix supervision; when the model's own \full{} behavior is already strong, increasing teacher scale brings limited additional benefit. Matched pollution tests and evidence-focus probes further show stronger grounding in completed user evidence and reduced sensitivity to assistant-side anchors.

 \end{itemize}

\section{Related Work}
\paragraph{Inference-time repair and control loops.}
Inference-time repair and self-correction methods address model errors by adding an explicit test-time control process around generation
\citep{pan2024automatically,kamoi2024selfcorrection}.
Representative approaches revise generations through verbal reflection or self-feedback
\citep{shinn2023reflexion,madaan2023selfrefine,ferraz2024decrim},
check candidate answers with self-verification, external critics, or learned verifiers
\citep{weng2023selfverification,gou2024critic,zhang2024strongverifiers,gao2023rarr},
or reset the context before regenerating
\citep{mohammad-khalid-etal-2025-ergo}.
These methods can be effective when extra inference passes and reliable
feedback are available, but they optimize an augmented test-time system.
Our setting is complementary: the model must answer from the completed
transcript itself, where earlier partial-information replies may remain
as misleading context. \ccopd{} targets this robustness through training,
rather than by adding a repair or reset loop at test time.

% \paragraph{Clarification, abstention, and online interaction control.} Another line of work studies how models should behave when a request is ambiguous or the information needed to answer is incomplete. Selective-prediction and ambiguity-aware generation methods study when models should withhold unreliable answers \citep{cole2023selective,kim2024ambiguity}. Clarification work evaluates when to ask a question and what question to ask \citep{zhang2024clamber,zhang2025clarify}. This literature asks whether to answer, ask, defer, or abstain under uncertainty. Our setting assumes that the missing evidence has already arrived, but the realized transcript may still contain assistant-side commitments made before it was complete. CCOPD targets how the model conditions on a completed but self-contaminated transcript.

\paragraph{Clarification, abstention, and online interaction control.}
Prior work studies when models should answer, clarify, defer, or
abstain under incomplete evidence
\citep{cole2023selective,kim2024ambiguity,zhang2024clamber,zhang2025clarify}.
RLAAR \citep{2026RLAAR} employs curriculum reinforcement
learning to improve task
completion and reduce premature answering in sharded conversations.
Our focus is complementary but distinct: rather than directly
optimizing decisions under incomplete information, \ccopd{}
targets final-answer reliability after all user evidence has
arrived, when earlier assistant-generated commitments may still
contaminate the completed transcript.

\paragraph{On-policy distillation (OPD).} 
% On-policy distillation reduces the trajectory mismatch in autoregressive distillation by supervising the student on states reached by its own policy, rather than only on fixed teacher demonstrations \citep{gu2024minillm,agarwal2024gkd}. Many recent OPD variants strengthen the supervision signal by giving the teacher access to additional task information \citep{zhao2026opsd,penaloza2026privileged,liu2026crosslingual,ye2026opcd,zhang2026opsdl}, or by coupling distillation with auxiliary inference, feedback, or optimization mechanisms \citep{zhao2026training,yang2026learning,zhu2026many}. 
On-policy distillation reduces trajectory mismatch by supervising
students on self-generated prefixes
\citep{gu2024minillm,agarwal2024gkd}.
Recent variants introduce privileged supervision
\citep{zhao2026opsd,penaloza2026pid}
or explore task-specific distillation and auxiliary training
mechanisms
\citep{liu2026crosslingual,zhao2026training,yang2026learning,zhu2026many}.
OPCD \citep{ye2026opcd} internalizes
contextual knowledge through on-policy reverse-KL distillation,
while OPSDL \citep{zhang2026opsdl} leverages relevant short
contexts to supervise long-context generation without adding
new task evidence.
Building on these cross-context distillation principles,
\ccopd{} addresses a distinct source of interference: the
model's own earlier replies generated before task completion. The teacher and student are given the same underlying task evidence and share the same backbone, but they condition on different presentations of that evidence: a clean FULL  for the teacher and a realized RAW-SHARDED transcript for the student. Thus, the teacher is presentation-privileged rather than information-privileged. The goal is not to import extra knowledge from a stronger or better-informed supervisor, but to distill invariance across task-equivalent presentations.

\section{Problem Statement}
\label{sec:problem}

\subsection{Problem Formulation}
\paragraph{Task-equivalent sharded presentations.}
For a task instance $\task$, let $\canon(\task)$ denote the canonical
\full{} prompt. Let $\hist(\task)$ denote a sharded presentation of the
same task, where task information is released across user turns and
interleaved with assistant replies before the final-answer request. We
call $\hist(\task)$ \emph{task-equivalent} to $\canon(\task)$ when the
sharded presentation reveals the same task-relevant user evidence as the
canonical prompt. 
% The model therefore has the evidence needed to solve
% the task, but receives it through a multi-turn conversational path rather
% than a single canonical prompt.

\paragraph{Canonical Context Consistency.}
For task-equivalent presentations, a reliable model should preserve its
final-answer distribution under the change of presentation:
\begin{equation}
    \pi(\yseq \mid \hist(\task))
    \approx
    \pi(\yseq \mid \canon(\task)).
    \label{eq:consistency}
\end{equation}
We call this requirement \emph{canonical-context consistency}. The
canonical prompt is not assumed to be an oracle; it is a controlled
presentation in which the same user evidence is available.

To inspect where this invariance breaks down, we compare next-token predictions under the two task-equivalent
presentations while holding the answer prefix fixed. For a shared answer
prefix $s$, define
\begin{equation}
\begin{aligned}
    \Psi_{\pi}(\task,s)
    =
    \KL\!\left(
    \pi(\cdot \mid \hist(\task),s)
    \;\middle\|\;
    \pi(\cdot \mid \canon(\task),s)
    \right).
\end{aligned}
\label{eq:local-presentation-gap}
\end{equation}
The same model appears on both sides; only the preceding presentation
changes. Hence $\Psi_{\pi}$ is a local diagnostic of presentation
sensitivity. Large values mark prefixes where the sharded path induces
continuation preferences that differ from those induced by the clean
presentation.

\paragraph{Self-anchored drift.}
The discrepancy in Eq.~\eqref{eq:local-presentation-gap} locates a shift
but does not identify which part of the sharded presentation is
responsible. In \raw{} histories, the distinctive source of concern is
earlier assistant text generated before the task evidence is complete.
We write a final-turn \raw{} history as
\begin{equation}
    \hist_{\raw}(\task)
    =
    (u_1,a_1,\ldots,u_{K-1},a_{K-1},u_K),
    \label{eq:raw-history}
\end{equation}
where $u_i$ are user shards and $a_i$ are non-final assistant replies.
Each $a_i$ is generated before all task evidence is available, and may
therefore contain unsupported guesses, provisional answers, or
task-specific commitments. After $u_K$ supplies the missing evidence, the
user evidence in the history is complete and task-equivalent to
$\canon(\task)$. However, the earlier $a_i$ remain in the conditioning
context. When these self-generated commitments pull the final-answer
behavior away from the completed user evidence, canonical-context
consistency fails. We call this failure mode \emph{self-anchored drift}.

\subsection{Probing self-anchored drift}
\label{sec:diagnostic-probes}

We probe the unmodified starting model to test whether non-final assistant
replies still affect the final predictive state after all user evidence
is present. We refer to these non-final replies as process replies.

\paragraph{SAAR and span edit.}
For completed \raw{} histories, we annotate completed user-evidence spans
$G_{\mathrm{usr}}$ and assistant commitment spans $G_{\mathrm{self}}$.
The \emph{Self-Anchor Attention Ratio} (SAAR) summarizes final-answer
attention over these spans:
\begin{equation}
    \mathrm{SAAR}_{\ell}
    =
    \log
    \frac{\bar{A}_{\ell}(G_{\mathrm{usr}})+\epsilon}
         {\bar{A}_{\ell}(G_{\mathrm{self}})+\epsilon}.
    \label{eq:saar-main}
\end{equation}
Here $\bar{A}_{\ell}(G)$ averages layer-$\ell$ attention from answer
tokens to group $G$; higher SAAR indicates denser attention to completed
user evidence relative to earlier assistant commitments. Using the same
spans, we also mask commitment spans and measure the change in the
gold-vs-process-anchor margin. \Cref{fig:base-self-anchor-problem}a
shows that this edit selectively helps states anchored to the model's
process-stage wrong answer.

\paragraph{Neutral-placeholder contrast.}
We replace each process reply with a length-matched neutral
acknowledgement, while keeping all user turns, role markers,
system prompts, and turn order fixed; the exact template and
construction are provided in Appendix~\ref{app:diagnostic-probes}. For prefix $i$, let
$C_i^{\mathrm{raw}}$, $C_i^{\mathrm{neu}}$, and $C_i^{\mathrm{full}}$
be the realized, neutralized, and canonical \full{} contexts. With
$p_{i,r}^M=p_M(\cdot\mid C_i^r)$ for
$r\in\{\mathrm{raw},\mathrm{neu}\}$ and
$q_i^{\mathrm{full}}=\base(\cdot\mid C_i^{\mathrm{full}})$, define
\begin{equation}
\Delta_i^M
=
\KL(p_{i,\mathrm{raw}}^M \,\|\, q_i^{\mathrm{full}})
-
\KL(p_{i,\mathrm{neu}}^M \,\|\, q_i^{\mathrm{full}}).
\label{eq:neutral-placeholder-contrast}
\end{equation}
Positive $\Delta_i^M$ means that process replies increase canonical
deviation; \Cref{fig:base-self-anchor-problem}b shows that neutralization
reduces it.

\begin{figure}[t]
\centering
\begin{subfigure}[t]{0.48\linewidth}
    \includegraphics[width=\linewidth]{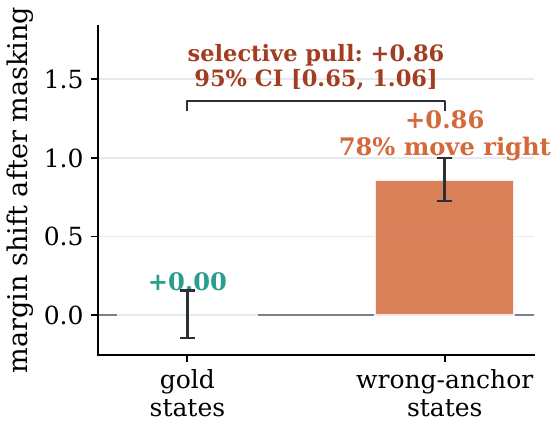}
   \caption{Span edit}
    \label{fig:base-self-anchor-left}
\end{subfigure}
\hfill
\begin{subfigure}[t]{0.48\linewidth}
    \includegraphics[width=\linewidth]{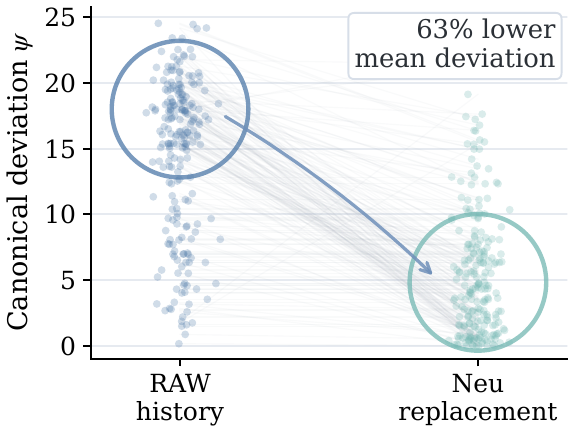}
    \caption{Neutral-placeholder replacement}
    \label{fig:base-self-anchor-right}
\end{subfigure}
\caption{
Qwen3-8b(Instruct) model probes for self-anchored drift.
(a) Masking process-reply commitment spans selectively improves
gold-vs-anchor margins for wrong-anchor states.
(b) Replacing process replies with neutral placeholders lowers
predictive-state deviation from the canonical \full{} reference.
}
\label{fig:base-self-anchor-problem}
\end{figure}

\section{Canonical-Context On-Policy Distillation}
\label{sec:method}
\begin{figure*}[t]
    \centering
    \includegraphics[width=\textwidth]{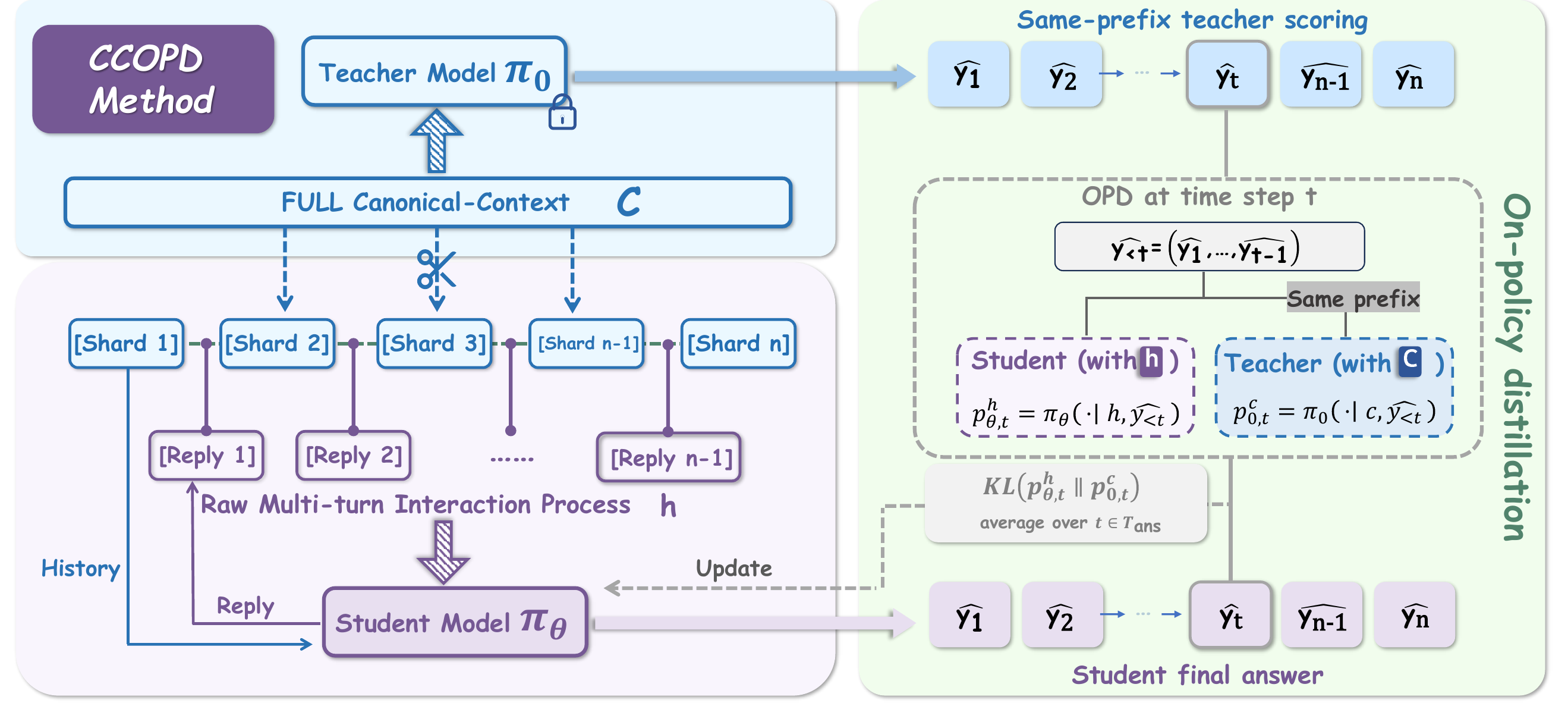}
    \caption{Overview method of \ccopd{}. CCOPD aligns a \raw{}-conditioned student with a frozen \full{}-conditioned teacher using answer-masked same-prefix reverse KL.}
    \label{fig:method}
\end{figure*}

\Cref{sec:problem} defines a local presentation gap at a shared answer
prefix. \ccopd{} turns this gap into an on-policy distillation signal:
a trainable student answers from the realized \raw{} history, while a
frozen copy of the same base model scores the student's own answer
prefixes under the canonical \full{} prompt. Training minimizes a
final-answer-masked reverse KL between these two next-token distributions.

\subsection{Retained Pairs and Model Roles}
\label{sec:method:pairs}

For each task $q$, we form a final-turn pair $(c,h)$,
where $c=c(q)$ is the original FULL prompt in its original
textual order, and $h$ is the retained RAW-SHARDED history
ending at the final user turn. In the main setting, $h=\hist_{\raw}(\task)$: all
task-relevant user shards have been revealed, and the non-final assistant
replies from the realized interaction remain in the context. We retain
only histories that end with the final user turn after the user evidence
is complete.

The two contexts are used on separate paths. The student is trainable and
conditions only on $h$; the canonical prompt $c$ is never appended to the
student input. The teacher is a frozen copy of the starting checkpoint and
conditions only on $c$. Thus, \ccopd{} compares two presentations of the
same user evidence while updating only the history-conditioned student.
Appendix~\ref{app:pairing} details the structural checks and leakage
controls.

\subsection{On-Policy Canonical Relabeling}
\label{sec:method:relabeling}

Given a retained pair $(c,h)$, the current student produces a
final-answer rollout
\begin{equation}
    \hat{y}_{1:T} \sim \student(\cdot \mid h),
    \label{eq:ccopd-rollout}
\end{equation}
and $T_{\mathrm{ans}}(\hat{y})$ denotes the token positions belonging to
that final assistant continuation.  These prefixes are the on-policy
states for distillation: they are reached after the student conditions
on the realized transcript, including any earlier assistant-side
commitments contained in $h$.

For each $t\in T_{\mathrm{ans}}(\hat{y})$, the student and teacher score
the same prefix $\hat{y}_{<t}$ under different contexts:
\begin{equation}
\label{eq:ccopd-dist}
p^{h}_{\theta,t}
= \student(\cdot \mid h, \hat{y}_{<t}),
\quad
p^{c}_{0,t}
= \base(\cdot \mid c, \hat{y}_{<t}).
\end{equation}
The teacher is not asked to decode a clean completion.  It only provides
next-token probabilities on the prefix that the student has already
visited.  Thus canonical relabeling changes the conditioning context of
the scorer, not the answer prefix being scored.

As in on-policy distillation, the rollout operation in
Equation~\eqref{eq:ccopd-rollout} selects the states but is treated as
stop-gradient.  Gradients flow through the next-token student
distribution evaluated at the selected prefixes.  
% In the main
% experiments, the expectation over rollouts is approximated with one
% temperature-zero rollout per retained pair, matching the decoding path
% used in evaluation.

\subsection{Answer-Masked Reverse-KL Objective}
\label{sec:method:objective}

Substituting Equation~\eqref{eq:ccopd-dist} into the canonical-deviation potential
gives the average deviation along the student rollout:
\begin{equation}
\begin{aligned}
    \bar{\Psi}_{\theta}(\task,\hat{y})
    =
    \frac{1}{|T_{\mathrm{ans}}(\hat{y})|}
    \sum_{t\in T_{\mathrm{ans}}(\hat{y})}
    \KL\!\left(
        p^{h}_{\theta,t}
        \;\middle\|\;
        p^{c}_{0,t}
    \right).
\end{aligned}
\label{eq:ccopd-potential-average}
\end{equation}
The implemented per-rollout loss is this quantity,
\begin{equation}
    \ell_{\mathrm{ccopd}}(\task,\hat{y})
    =
    \bar{\Psi}_{\theta}(\task,\hat{y}),
    \label{eq:ccopd-example}
\end{equation}
and the ideal objective is
\begin{equation}
\begin{aligned}
    \mathcal{L}_{\mathrm{ccopd}}(\theta)
    =
    \mathbb{E}_{(c,h)\sim\mathcal{D}_{\mathrm{pair}}}
    \;\mathbb{E}_{\hat{y}\sim\student(\cdot\mid h)}
    \left[
        \ell_{\mathrm{ccopd}}(\task,\hat{y})
    \right].
\end{aligned}
\label{eq:ccopd}
\end{equation}
The finite training loss estimates Equation~\eqref{eq:ccopd} by
averaging Equation~\eqref{eq:ccopd-example} over minibatches.
Normalizing by $|T_{\mathrm{ans}}(\hat{y})|$ prevents long answers from
dominating the objective.

The final-answer mask is part of the objective, not an implementation
detail.  Intermediate assistant turns in a \raw{} conversation are
produced before all user evidence is available; forcing those turns to
match a \full{}-conditioned teacher would leak future evidence and
penalize reasonable partial-information behavior.  CCOPD therefore
optimizes canonical alignment only after the final user turn, where the
history and the canonical prompt are task-equivalent.

\paragraph{Sequence-level view.}
Let $P^{h}_{\theta}$ be the distribution over complete final answers
induced by the student under $h$, and let $P^{c}_{0}$ be the analogous
distribution induced by the frozen base model under $c$.  Under an ideal
stochastic formulation, the autoregressive chain rule decomposes
$\KL(P^{h}_{\theta}\|P^{c}_{0})$ into the expected sum of the same-prefix
token KL terms along student-sampled trajectories.  Equation~\eqref{eq:ccopd}
can therefore be read as a practical, answer-masked estimator of
answer-level alignment to the canonical full-context behavior.  Appendix~\ref{app:proof} states the exact bridge and its assumptions.

% \subsection{Algorithm and Design Rationale}
% \label{sec:method:algorithm}

\Cref{fig:method} shows training procedure.  The important
asymmetry is that the student path and teacher path receive different
contexts but score the same prefix.  This yields a signal that is both
cross-context and on-policy.

% \begin{algorithm}[h]

% \caption{One \ccopd{} training step}
% \label{alg:ccopd}
% \begin{algorithmic}[1]
% \Require canonical prompt $c$, final-turn history $h$, trainable student $\student$, frozen teacher $\base$
% \State Generate rollout $\hat{y}_{1:T} \sim \student(\cdot \mid h)$
% \State Identify final-answer positions $T_{\mathrm{ans}}(\hat{y})$
% \For{$t \in T_{\mathrm{ans}}(\hat{y})$}
%     \State Compute $p^{h}_{\theta,t}=\student(\cdot \mid h,\hat{y}_{<t})$
%     \State Compute $p^{c}_{0,t}=\base(\cdot \mid c,\hat{y}_{<t})$
%     \State Accumulate $\KL(p^{h}_{\theta,t}\|p^{c}_{0,t})$
% \EndFor
% \State Update only the student using the averaged KL loss
% \end{algorithmic}
% \end{algorithm}

% Experiment section for an ACL/EMNLP two-column main paper.
% Requires: booktabs, graphicx, amsmath, natbib, cleveref. The main table intentionally omits WAIT.

\providecommand{\ccopd}{\textsc{CCOPD}}
\providecommand{\full}{\textsc{Full}}
\providecommand{\concat}{\textsc{Concat}}
\providecommand{\raw}{\textsc{Raw-Sharded}}
\providecommand{\waittraj}{\textsc{Wait}}
\providecommand{\useronly}{\textsc{User-Only}}
\providecommand{\NA}{--}
\providecommand{\best}[1]{\textbf{#1}}
\providecommand{\ccopd}{\textsc{CCOPD}}
\providecommand{\full}{\textsc{Full}}
\providecommand{\concat}{\textsc{Concat}}
\providecommand{\waittraj}{\textsc{Wait}}
\providecommand{\useronly}{\textsc{User-Only}}
\providecommand{\raw}{\textsc{Raw-Sharded}}
\providecommand{\na}{--}
\providecommand{\pending}{\textit{pending}}
\providecommand{\fcr}[3]{#1\,/\,#2\,/\,#3}
\providecommand{\best}[1]{\textbf{#1}}
\providecommand{\grc}{\textsc{GRC}}
\providecommand{\cgdc}{\textsc{CGDC}}
\newcommand{\finding}[2]{%
\begin{tcolorbox}[
    colback=gray!6,
    colframe=teal!60!black,
    arc=4pt,
    boxrule=0.7pt,
    left=1mm,
    right=1mm,
    top=1mm,
    bottom=1mm,
    enhanced,
    drop shadow
]
\noindent\textbf{\textit{Finding #1:}}~\textit{#2}
\end{tcolorbox}%
}

\section{Experiments and Analysis}
\label{sec:experiments}

We evaluate whether \ccopd{} improves canonical-context consistency after the user evidence is complete. The target setting is RAW-SHARDED. FULL uses the original complete
prompt, whereas CONCAT places all user shards into one user message
in exactly the same shard order as RAW-SHARDED, but without
intermediate assistant replies. 

% \textbf{RQ1: Completed-interaction accuracy.}
% Once all task-relevant instructions have been revealed, does \ccopd{} improve final-answer performance in multi-turn interactions while preserving performance when the same complete instruction is given in a clean single-turn prompt?

% \textbf{RQ2: Cross-domain transfer}.When complete instructions are revealed gradually across turns, does a model post-trained with \ccopd{} on a single source domain improve final-answer performance in held-out task families without target-domain training examples?

% \textbf{RQ3: Evidence grounding under contaminated histories.}
% Are \ccopd{}'s gains associated with stronger grounding in the completed user evidence and reduced sensitivity to assistant-side commitments, rather than generic post-training or history-cleaning effects?

% We first describe the protocol, including the fixed-history diagnostics used for RQ3, and then report task results followed by mechanism-oriented audits.

\subsection{Experimental Setup}
\label{sec:exp-setup}

\paragraph{Models.}
All model runs use instruction-tuned checkpoints. Our primary experiments use Qwen3-8B(Instruct) with a LoRA adapter trained by the reverse-KL \ccopd{} objective in \cref{eq:ccopd} \citep{yang2025qwen3,hu2022lora}. For the main Qwen3-8B run, we optimize the LoRA adapter with AdamW using a learning rate of $3\times10^{-5}$; training configurations and dynamics are reported in
Appendix~\ref{app:training_details}. To assess scale and family robustness, we additionally train adapters for Qwen3-4B and Llama3.1-8B-Instruct \citep{llama3-herd-2024} in the scaling analysis.

\paragraph{Training data and pair construction.}
Our primary training uses sharded math conversations
from GSM8K and GSM8K-Aug
\citep{cobbe2021gsm8k,deng2024explicit}.
For each backbone, we generate on-policy \raw{}
interactions from the same shard specification.
The primary Qwen3-8B adapter is trained on 6k retained
pairs after screening. Additional experiments examine
CCOPD with DAPO-Math-17K \citep{dapo2025} to assess training-data
generality, and with HotpotQA \citep{yang2018hotpotqa} to study non-math-to-math
transfer (Appendix~\ref{app:additional_checks}).

\paragraph{Evaluation protocol.}
We follow the released sharded-instruction protocol from Lost in Conversation \citep{laban2025lost}. The evaluation suite covers six task families: math word problems from GSM8K-style examples \citep{cobbe2021gsm8k}, code generation from HumanEval and LiveCodeBench-derived tests \citep{chen2021codex,jain2024livecodebench}, function calling from BFCL-style tasks \citep{patil2025bfcl}, text-to-SQL from Spider \citep{yu2018spider}, table-to-text generation from ToTTo \citep{parikh2020totto}, and long-context summarization from SummHay \citep{laban2024summhay}. The non-math tasks are zero-shot transfer settings for the math-trained adapter. We additionally evaluate AIME~2025
as a harder mathematical reasoning stress test
(Appendix~\ref{app:additional_checks}).
We report results averaged over ten independently
sampled evaluation trajectories per task instance.
User shards and their order remain fixed, while
intermediate assistant replies and final responses
are regenerated with different decoding seeds.
These repetitions represent evaluation rollouts.
Further details are provided in Appendix~\ref{app:raw-shard-eval}.

\begin{table*}[t]
\centering
\scriptsize
\setlength{\tabcolsep}{2.55pt}
\renewcommand{\arraystretch}{1.08}
\resizebox{\textwidth}{!}{%
\begin{tabular}{@{}lcccccc@{}}
\toprule
Model / variant & Math F/C/R & Structured OOD F/C/R & Generation OOD F/C/R & All F/C & OOD-R & All-R \\
\midrule
\rowcolor{gray!10}\multicolumn{7}{@{}l}{\textit{Model-size and model-family transfer}} \\
Qwen3-14B (Instruct) & \fcr{96.0}{92.9}{71.7} & \fcr{92.1}{83.1}{52.8} & \fcr{34.1}{31.9}{20.2} & 70.3 & 39.6 & 44.9 \\
Qwen3-4B (Instruct) & \fcr{89.2}{84.4}{55.2} & \fcr{80.0}{74.2}{41.0} & \fcr{22.4}{22.6}{11.5} & 60.2 & 29.1 & 33.4 \\
\rowcolor{blue!6}Qwen3-4B \ccopd{} & \fcr{89.2}{83.5}{56.9 (+1.7)} & \fcr{83.8}{75.9}{63.8} & \fcr{21.0}{18.5}{17.2} & 60.6 (+0.4) & \best{45.0 (+15.9)} & \best{46.9 (+13.6)} \\
Llama3.1-8B-Instruct & \fcr{71.7}{71.0}{46.5} & \fcr{67.6}{63.8}{37.4} & \fcr{13.0}{16.3}{13.1} & 49.4 & 27.6 & 30.7 \\
\rowcolor{blue!6}Llama3.1-8B \ccopd{} & \fcr{74.7}{73.7}{53.5 (+7.0)} & \fcr{67.5}{64.0}{52.8} & \fcr{14.8}{13.8}{13.4} & 49.7 (+0.4) & \best{36.9 (+9.3)} & \best{39.6 (+8.9)} \\
\midrule
\rowcolor{gray!10}\multicolumn{7}{@{}l}{\textit{Qwen3-8B: objectives and teacher sources}} \\
Qwen3-8B Forward-\ccopd{} & \fcr{93.3}{91.2}{79.0 (+13.1)} & \fcr{84.6}{76.0}{65.5} & \fcr{26.9}{23.3}{20.4} & 63.6 (-0.8) & 47.3 (+10.4) & 52.5 (+10.8) \\
Qwen3-8B 4B-teacher \ccopd{} & \fcr{90.2}{88.4}{67.9 (+1.9)} & \fcr{85.7}{78.8}{67.6} & \fcr{26.4}{23.1}{21.0} & 64.0 (-0.5) & 48.7 (+11.8) & 51.9 (+10.2) \\
Qwen3-8B 14B-teacher \ccopd{} & \fcr{92.3}{86.5}{69.9 (+4.0)} & \fcr{84.6}{78.4}{66.5} & \fcr{27.7}{24.0}{20.5} & 64.0 (-0.4) & 47.9 (+11.0) & 51.5 (+9.9) \\
\midrule
\rowcolor{gray!10}\multicolumn{7}{@{}l}{\textit{Qwen3-8B: complete baselines}} \\
Qwen3-8B (Instruct) & \fcr{90.2}{87.4}{65.9} & \fcr{85.4}{78.8}{48.8} & \fcr{28.4}{24.8}{19.4} & 64.4 & 36.9 & 41.7 \\
Qwen3-8B SFT & \fcr{93.3}{91.2}{65.9 (+0.0)} & \fcr{86.8}{77.1}{64.5} & \fcr{27.3}{22.4}{18.9} & 64.3 (-0.1) & 46.1 (+9.2) & 49.3 (+7.7) \\
Qwen3-8B OPSD & \fcr{86.0}{85.0}{56.2 (-9.7)} & \fcr{87.8}{77.8}{54.9} & \fcr{27.6}{23.8}{17.8} & 63.9 (-0.5) & 39.9 (+3.0) & 42.6 (+0.9) \\
Qwen3-8B GRPO & \fcr{91.3}{90.2}{74.8 (+8.8)} & \fcr{86.9}{83.9}{45.7} & \fcr{28.4}{24.8}{19.2} & 66.4 (+2.0) & 34.9 (-1.9) & 41.5 (-0.2) \\
\rowcolor{blue!14}Qwen3-8B \ccopd{} & \fcr{90.2}{88.4}{\best{82.5 (+16.6)}} & \fcr{86.4}{78.0}{\best{68.0}} & \fcr{27.9}{23.3}{\best{22.9}} & 64.2 (-0.2) & \best{49.8 (+12.9)} & \best{55.1 (+13.5)} \\
\bottomrule
\end{tabular}%
}
\caption{Master results with weighted aggregate scores. F/C/R denotes \full{}/\concat{}/\raw{}. Structured OOD averages code, function calling, and text-to-SQL; Generation OOD averages table-to-text and summarization. All F/C is a single clean-presentation score obtained by averaging the weighted \full{} and weighted \concat{} scores across all six task families. OOD-R averages non-math \raw{} scores; All-R averages all six task families. Parenthesized values indicate changes against the matching base row under the same weighted scope; in Math F/C/R, only the R component carries the parenthesized RAW delta.}
\label{tab:master-results}
\end{table*}

\paragraph{Compared systems.}
Baseline denotes the unmodified instruction-tuned checkpoint
and is the reference for reported deltas.
For data-matched training comparisons, \textbf{SFT} uses the same retained
final-turn histories $h$ and the same final-answer mask as \ccopd{}, but
replaces distributional alignment with supervised final-answer targets;
\textbf{GRPO} uses the same math source and final-answer reward/evaluation
protocol. \textbf{OPSD} is a related on-policy
distillation comparison with privileged supervision.
\textbf{Forward-\ccopd{}} is a token-level control that keeps the retained
histories, final-answer positions, and student prefixes fixed, while
changing only the KL direction or teacher conditioning context.
\textbf{Teacher-source} rows replace the same-backbone \full{} teacher
with smaller or larger teachers.

\subsection{Main Results}
\label{sec:exp-main}
\Cref{tab:master-results} summarizes the main results; \Cref{tab:ood-fcr-details} gives per-domain details.

\paragraph{RAW improves without clean-task drift.}
On Qwen3-8B, \ccopd{} raises math \raw{} accuracy from 65.9 to 82.5
(+16.6 points), while math \full{} remains unchanged at 90.2.
\ccopd{} also outperforms SFT and GRPO on math \raw{}, suggesting that
the observed \raw{} gains are not recovered by standard post-training
objectives alone.

\paragraph{Math-only training transfers out of domain.}
The same math-trained adapter improves the non-math \raw{} average from
36.9 to 49.8 (+12.9 points). Structured OOD \raw{} rises from 48.8 to
68.0 (+19.2), and generation OOD \raw{} rises from 19.4 to 22.9 (+3.5).
Because training uses only math conversations, this pattern is more
consistent with transferable history grounding than with a math-format
heuristic. Qwen3-4B and Llama3.1-8B-Instruct show the same aggregate \raw{}
direction.

% We further train \ccopd{} on non-math HotpotQA\cite{yang2018hotpotqa} sharded histories and still
% observe an improvement on math \raw{} ,
% suggesting that the signal is not tied to math-format supervision alone;
% details are in Appendix~\ref{app:source-domain}.

Additional checks extend the training source and evaluation
difficulty: HotpotQA training yields 76.7 math RAW-SHARDED
accuracy, DAPO-Math-17K training reaches 56.6 All-R, and an
AIME~2025 stress test improves RAW-SHARDED accuracy from 3.3
to 10.0. Full results and scope are reported in
Appendix~\ref{app:additional_checks}.

%还需要一个find
% \finding{3}{Generic baselines do not explain the pattern.}
% Base exposes the gap; SFT improves some OOD rows but leaves math \raw{} below the base; OPSD-style distillation is weaker on the available OOD average; GRPO is strong on the partial math-plus-structured scope but lacks generation OOD coverage. Among complete same-backbone rows, \ccopd{} gives the best All-R while preserving clean presentations. Larger-teacher rows are competitive but change the method into cross-model distillation.

\begin{figure*}[t]
\centering
\begin{minipage}[t]{0.32\linewidth}
    \centering
    \includegraphics[width=\linewidth]{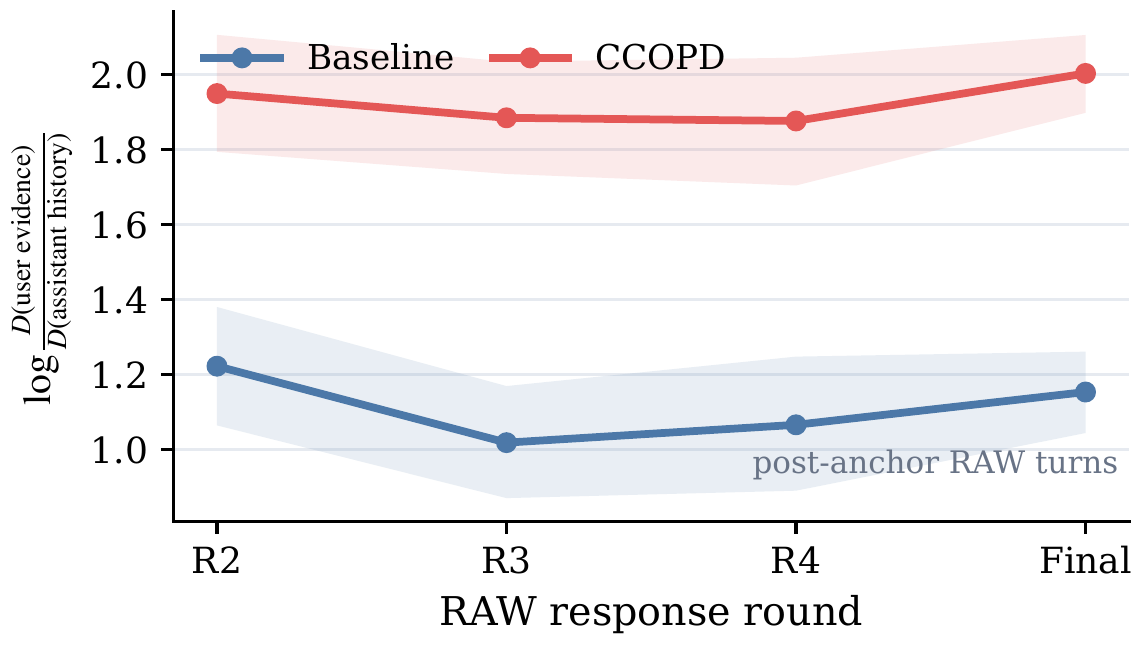}
    \centerline{\scriptsize (a) Turn-level evidence focus}
\end{minipage}
\hfill
\begin{minipage}[t]{0.32\linewidth}
    \centering
    \includegraphics[width=\linewidth]{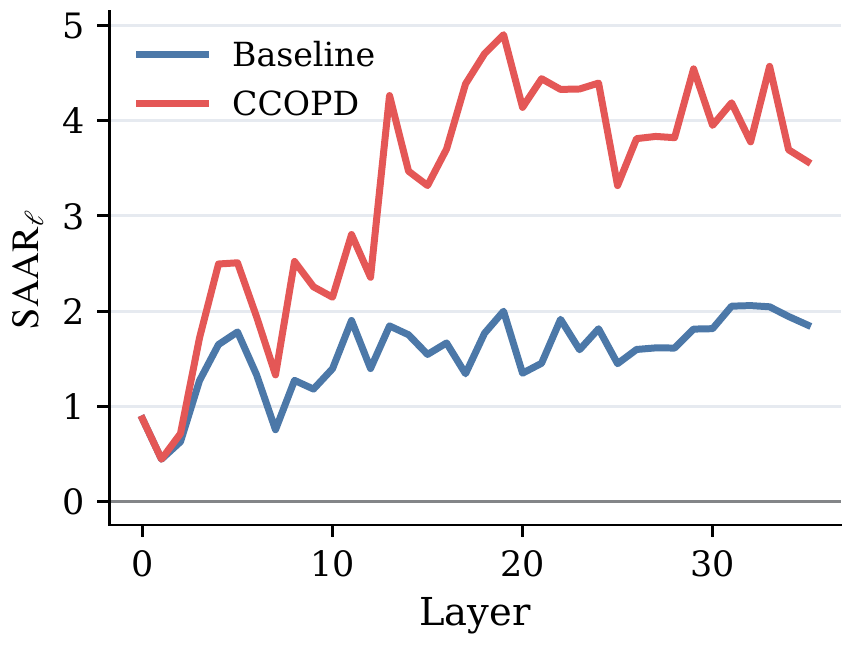}
    \centerline{\scriptsize (b) Final-state SAAR}
\end{minipage}
\hfill
\begin{minipage}[t]{0.32\linewidth}
    \centering
    \includegraphics[width=\linewidth]{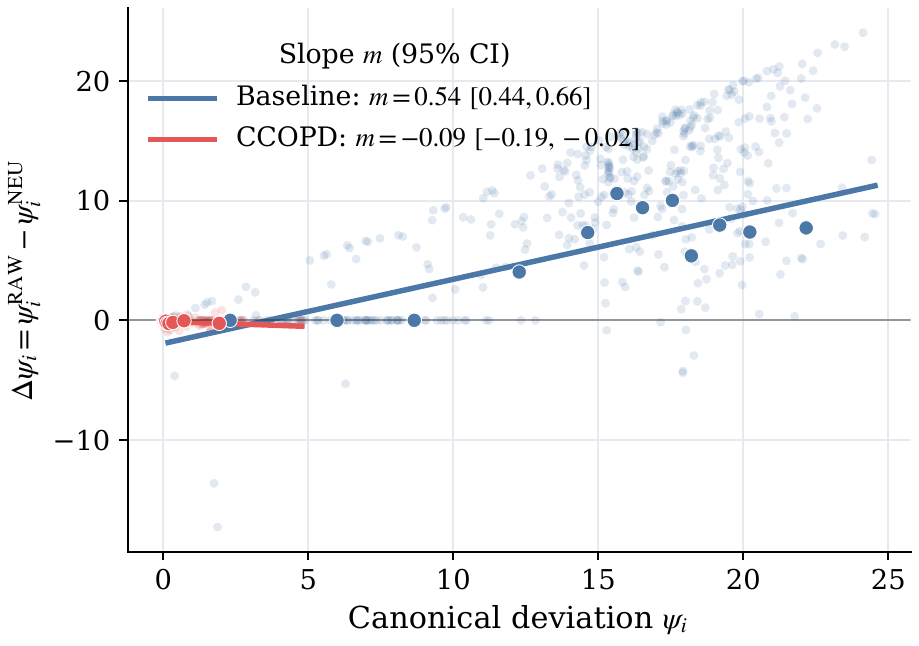}
    \centerline{\scriptsize (c) Neutral replacement}
\end{minipage}

\caption{
Mechanism diagnostics for self-anchored drift.
(a) \ccopd{} maintains stronger user-evidence focus across \raw{}
response rounds after process replies enter the context.
(b--c) At matched final-answer prefixes, \ccopd{} increases SAAR and
flattens assistant-associated canonical deviation, reducing the
neutral-replacement slope from $0.537$ to $-0.087$.
}
\label{fig:mechanism-diagnostics}

\end{figure*}
\subsection{Ablations and Mechanism Diagnostics}

\paragraph{Teacher source.}
Teacher-source ablations in Table~1 suggest that scale is not the main
bottleneck. A smaller 4B \full{} teacher still improves the weighted
\raw{} aggregate over the unmodified baseline, while a larger 14B teacher improves
over the base but remains below the same-backbone teacher. This pattern
supports the intended role of the teacher in \ccopd{}: it need not supply
new task knowledge, but should provide a compatible canonical presentation
of the same evidence. Cross-model mismatch may partly limit external
teachers, especially at larger scale.

\begin{table}[t]
\centering
\small
\setlength{\tabcolsep}{4pt}
\renewcommand{\arraystretch}{1.05}
\begin{tabular}{lcc}
\toprule
Condition & Baseline  & \ccopd{}  \\
\midrule
Clean \full{} problem & 90.2 & 90.2 \\
+ Assistant-side wrong solution & 33.0 & 89.3 \\
+ User-side wrong answer hint & 63.1 & 88.3 \\
\bottomrule
\end{tabular}
\caption{
Full-context pollution stress test on math.
}
\label{tab:full-context-pollution}
\end{table}

\paragraph{KL direction.}
Forward-CCOPD keeps the student rollouts, \full{} teacher, final-answer
mask, and same-prefix scoring fixed, but replaces reverse KL with forward
KL. This control tests both whether fine-grained same-prefix canonical
supervision helps over the unmodified baseline and whether the KL direction
matters beyond that supervision. Forward-CCOPD improves substantially
over the base, but remains below reverse-KL \ccopd{}. We attribute the gap
to the objectives: forward KL encourages the \raw{} student to cover the
\full{} teacher's likely continuations, whereas reverse KL penalizes
off-canonical probability mass assigned by the deployed \raw{} policy.
Because self-anchored drift appears in the policy used at inference time,
reverse KL better matches the intended intervention.

\paragraph{History-stress diagnostics.}
We test grounding under matched polluted histories. Starting from full
math problems, we insert the same per-example wrong numeric anchor either
as an assistant-side prior solution or as a user-side answer hint; within
each condition, both models receive the same evaluated history.
\Cref{tab:full-context-pollution} shows that the baseline drops from 90.2 to 33.0/63.1
under assistant-/user-side pollution, whereas CCOPD remains
at 89.3/88.3. These differently formatted perturbations are
not a role-matched trust test. A separate assistant-checkpoint
test reports 82.5\% accuracy on both the original and augmented
tasks, providing a narrow check of useful assistant-history use
(Appendix~\ref{app:assistant_checkpoint}).

\paragraph{Evidence focus across process and final-answer states.}
We separate trajectory-level evidence focus from final-state anchoring in
\Cref{fig:mechanism-diagnostics}. At the \raw{} trajectory level,
\Cref{fig:mechanism-diagnostics}a tracks attention allocation across
response rounds: after process replies enter the history, \ccopd{}
maintains higher relative attention to task-relevant user evidence than
to process history. At completed final-answer states,
\Cref{fig:mechanism-diagnostics}b--c apply the same-prefix diagnostics
from \Cref{sec:problem}: SAAR compares completed user-evidence spans with
process-reply commitment spans, and neutral replacement measures
assistant-associated canonical deviation under fixed user shards.
\ccopd{} raises  layer SAAR and flattens the
neutral-replacement lower-envelope slope from $0.537$ to $-0.087$.
Together, these probes are consistent with reduced self-anchored drift:
the trained model attends more to completed user evidence and shows less
assistant-associated deviation at the final answer.

% \paragraph{Lightweight test-time modes.}
% We additionally evaluate two lightweight test-time modes: reset-then-answer (RTA)
% and defer-until-complete (DUC). As shown in \Cref{tab:lightweight-modes}, direct
% \ccopd{} is the strongest setting on math \raw{}. RTA helps
% the base model, whereas DUC has little effect. Applying
% the same modes to \ccopd{} slightly lowers accuracy, although both variants
% remain above the corresponding base+mode rows. This pattern is consistent
% with the possibility that \ccopd{} has already internalized part of the
% task-state management that these prompts impose explicitly. Implementation
% details are provided in Appendix~\ref{app:mode-details}.

\paragraph{Comparison with inference-time controls.}
We compare \ccopd{} with two lightweight test-time
prompting baselines: reset-then-answer (RTA) and
defer-until-complete (DUC). As shown in
Table~\ref{tab:lightweight-modes}, RTA and DUC improve
baseline accuracy from 65.9 to 72.8 and 67.0,
respectively. In contrast, \ccopd{} achieves 82.5
without additional inference-time intervention,
outperforming the strongest prompting baseline by
9.7 points. Additional analyses of their interactions
with \ccopd{} are provided in
Appendix~\ref{app:mode-details}.

\begin{table}[t]
\centering
\begin{tabular}{lc}
\toprule
Method & Math RAW Accuracy (\%) \\
\midrule
Baseline         & 65.9 \\
Baseline + RTA   & 72.8 \\
Baseline + DUC   & 67.0 \\
\midrule
\ccopd{}         & \textbf{82.5} \\
\bottomrule
\end{tabular}
\caption{Comparison with lightweight inference-time
prompting baselines on math \raw{}. \ccopd{} requires
no additional test-time intervention.}
\label{tab:lightweight-modes}
\end{table}

\FloatBarrier
\section{Conclusion}

This paper reframes \raw{} multi-turn failure as self-anchored canonical-context drift. The final history may contain all user evidence, but earlier assistant assumptions can still pull the model away from the answer distribution it would use under a clean \full{} prompt. \ccopd{} addresses this failure by distilling \raw{} final-answer states toward a frozen \full{} teacher on the same generated prefix. The method adds no inference-time component, improves on-policy sharded accuracy after math-only training, and transfers zero-shot to five non-math task types. Attention redistribution and WAIT replacement diagnostics suggest that the trained model becomes less sensitive to self-generated assumptions and more aligned with user evidence at the final answer.

\section*{Limitations}
Our experiments are still moderate in scale. The primary trained model is an 8B open-weight backbone with a LoRA adapter, trained on thousands rather than millions of retained canonical/history pairs. While this scale is sufficient to test the proposed mechanism, it does not establish how CCOPD behaves for substantially larger models, closed-source systems, or broader post-training corpora.
Our RAW-SHARDED setting is also a controlled approximation of multi-turn interaction. It is constructed by splitting complete benchmark tasks into shards, which enables a clean task-equivalent comparison against FULL and CONCAT prompts. However, real conversations may contain evolving intents, clarification requests, irrelevant side information, user corrections, and failures that cannot be reduced to deterministic shards of a complete prompt. Future work should collect organic multi-turn failures, construct human-verified canonical contexts that make those tasks answerable, and use such paired data to study self-anchored drift in more natural settings. Finally, CCOPD depends on a useful FULL-conditioned reference and only optimizes the final-answer state after all user evidence is available. It can improve consistency with the model’s clean full-context behavior, but it cannot correct errors made by the FULL teacher itself, nor does it directly train earlier-turn behaviors such as clarification, abstention, or avoiding premature commitments under incomplete information. The
measured gains concern final-answer robustness after the
evidence is complete. 
Intermediate assistant replies are not directly supervised
or independently evaluated for quality in our experiments.
 The assistant-checkpoint test shows that
required assistant-side information can still be used in a
controlled case, while calibrated use of helpful versus
misleading assistant content in more natural interactions
remains open.

\bibliography{cited_references}

\appendix
% Required packages:
% \usepackage{booktabs,graphicx}
% \usepackage[table]{xcolor}
% Assumed macros: \ccopd, \full, \concat, \raw.
\providecommand{\ccopd}{\textsc{CCOPD}}
\providecommand{\full}{\textsc{Full}}
\providecommand{\concat}{\textsc{Concat}}
\providecommand{\raw}{\textsc{Raw-Sharded}}
\providecommand{\na}{--}
\providecommand{\fcr}[3]{#1/#2/#3}
\providecommand{\best}[1]{\textbf{#1}}

\begin{table*}[t]
\centering
\scriptsize
\setlength{\tabcolsep}{2.6pt}
\renewcommand{\arraystretch}{1.08}
\resizebox{\textwidth}{!}{%
\begin{tabular}{@{}llccccc@{}}
\toprule
Model & Variant & Code F/C/R & Function F/C/R & Text-to-SQL F/C/R & Table-to-text F/C/R & Summary F/C/R \\
\midrule
\rowcolor{gray!10}\multicolumn{7}{@{}l}{\textit{Model-size and model-family transfer}} \\
Qwen3-14B(Instruct) & Baseline & \fcr{89.3}{71.6}{47.5} & \fcr{98.0}{99.1}{55.3} & \fcr{89.0}{78.0}{55.2} & \fcr{45.1}{43.2}{28.3} & \fcr{19.7}{17.1}{9.6} \\
Qwen3-4B(Instruct) & Baseline & \fcr{63.8}{50.4}{25.7} & \fcr{94.2}{99.1}{46.8} & \fcr{81.2}{72.1}{49.6} & \fcr{28.6}{31.8}{16.0} & \fcr{14.3}{10.7}{5.6} \\
\rowcolor{blue!6}Qwen3-4B & \ccopd{} & \fcr{74.5}{52.7}{38.2} & \fcr{95.3}{99.0}{81.9} & \fcr{81.2}{74.8}{70.0} & \fcr{27.1}{25.2}{22.5} & \fcr{13.1}{9.8}{10.3} \\
Llama3.1-8B-Instruct & Baseline & \fcr{38.6}{32.3}{16.8} & \fcr{86.8}{88.3}{56.2} & \fcr{75.8}{69.3}{38.2} & \fcr{15.2}{22.0}{17.0} & \fcr{10.1}{8.9}{8.1} \\
\rowcolor{blue!6}Llama3.1-8B & \ccopd{} & \fcr{38.4}{33.5}{28.3} & \fcr{86.8}{86.8}{\best{85.1}} & \fcr{75.8}{70.1}{44.0} & \fcr{16.6}{19.4}{17.3} & \fcr{12.5}{6.5}{8.2} \\
\midrule
\rowcolor{gray!10}\multicolumn{7}{@{}l}{\textit{Qwen3-8B: objective and teacher-source ablations}} \\
Qwen3-8B & Forward-\ccopd{} & \fcr{73.7}{56.2}{46.4} & \fcr{96.3}{98.0}{79.1} & \fcr{83.3}{73.0}{70.0} & \fcr{36.3}{33.0}{27.5} & \fcr{14.7}{10.6}{11.2} \\
Qwen3-8B & 4B teacher & \fcr{75.3}{61.8}{45.6} & \fcr{97.1}{98.0}{80.1} & \fcr{84.1}{75.7}{\best{75.8}} & \fcr{34.7}{32.3}{29.0} & \fcr{15.6}{11.2}{10.6} \\
Qwen3-8B & 14B teacher & \fcr{69.8}{59.4}{46.5} & \fcr{98.0}{99.1}{84.8} & \fcr{85.1}{75.7}{67.2} & \fcr{37.7}{34.4}{27.0} & \fcr{14.7}{10.4}{12.1} \\
\midrule
\rowcolor{gray!10}\multicolumn{7}{@{}l}{\textit{Qwen3-8B: complete evaluation and Baselines}} \\
Qwen3-8B(Instruct) & Baseline & \fcr{70.4}{61.7}{33.2} & \fcr{97.0}{99.1}{58.0} & \fcr{88.0}{74.7}{54.3} & \fcr{38.3}{35.7}{23.8} & \fcr{15.4}{10.5}{13.6} \\
Qwen3-8B & SFT & \fcr{74.7}{60.5}{39.3} & \fcr{98.0}{98.0}{84.7} & \fcr{87.0}{72.0}{68.3} & \fcr{36.0}{31.5}{23.7} & \fcr{16.0}{10.6}{12.7} \\
Qwen3-8B & OPSD & \fcr{76.8}{54.6}{35.4} & \fcr{98.1}{99.1}{67.5} & \fcr{88.0}{78.6}{60.8} & \fcr{38.8}{36.1}{24.2} & \fcr{13.1}{7.7}{9.4} \\
Qwen3-8B & GRPO & \fcr{75.2}{74.8}{32.5} & \fcr{97.0}{99.1}{54.2} & \fcr{87.9}{77.6}{49.6} & \fcr{38.1}{35.5}{24.1} & \fcr{15.8}{10.8}{12.7} \\
\rowcolor{blue!14}Qwen3-8B & \ccopd{} & \fcr{75.5}{60.3}{\best{49.4}} & \fcr{98.0}{99.1}{83.7} & \fcr{85.1}{73.7}{70.0} & \fcr{37.4}{32.9}{\best{29.4}} & \fcr{15.6}{10.7}{\best{14.4}} \\
\bottomrule
\end{tabular}%
}
\caption{OOD task details. Each cell reports \full{}/\concat{}/\raw{} (F/C/R).  }
\label{tab:ood-fcr-details}
\end{table*}

\section{Proof Details and Implementation Notes for \ccopd{}}
\label{app:proof}

\subsection{Answer-Event Control}
\label{app:answer-event-control}

Fix a task instance \(\task\). Write \(h=\hist(\task)\) for the completed multi-turn history and \(c=\canon(\task)\) for the canonical \full{} prompt. Let \(P^h_\theta\) be the distribution over complete final-answer strings induced by \(\student(\cdot\mid h)\), and let \(P^c_0\) be the corresponding distribution induced by \(\base(\cdot\mid c)\). We include EOS in the answer vocabulary and map any sequence that reaches the generation budget before EOS to a distinguished truncation symbol. Under this convention, both distributions are defined over the same finite terminal-answer space \(\mathcal{A}\). We also assume \(P^h_\theta\ll P^c_0\) on \(\mathcal{A}\), so the KL terms below are well defined.

\begin{lemma}[On-policy chain rule]
\label{lem:on-policy-chain-rule}
For \(y\sim P^h_\theta\) with terminal length \(\tau(y)\),
\begin{equation}
\begin{aligned}
    \KL(P^h_\theta\|P^c_0)
    =
    \mathbb{E}_{y\sim P^h_\theta}
    \left[
        \sum_{t=1}^{\tau(y)}
        d_t(y)
    \right],
\end{aligned}
\label{eq:appendix-chain-rule}
\end{equation}
where
\begin{equation}
    d_t(y)
    =
    \KL\!\left(
        \student(\cdot\mid h,y_{<t})
        \middle\|
        \base(\cdot\mid c,y_{<t})
    \right).
\label{eq:appendix-dt}
\end{equation}
\end{lemma}

\begin{proof}
For any terminal answer string \(y=(y_1,\ldots,y_{\tau(y)})\in\mathcal{A}\), the two autoregressive distributions factor as
\begin{align}
    P^h_\theta(y)
    &=
    \prod_{t=1}^{\tau(y)}
    \student(y_t\mid h,y_{<t}),
    \label{eq:appendix-student-factor}\\
    P^c_0(y)
    &=
    \prod_{t=1}^{\tau(y)}
    \base(y_t\mid c,y_{<t}).
    \label{eq:appendix-teacher-factor}
\end{align}
Substituting these factorizations into sequence-level KL gives
\begin{equation}
\begin{aligned}
    &\KL(P^h_\theta\|P^c_0)
    =
    \sum_{y\in\mathcal{A}}
    P^h_\theta(y)
    \log\frac{P^h_\theta(y)}{P^c_0(y)} \\
    &=
    \mathbb{E}_{y\sim P^h_\theta}
    \left[
        \sum_{t=1}^{\tau(y)}
        \log
        \frac{
            \student(y_t\mid h,y_{<t})
        }{
            \base(y_t\mid c,y_{<t})
        }
    \right].
\end{aligned}
\label{eq:appendix-sequence-kl-expand}
\end{equation}
Conditioning on a realized prefix \(y_{<t}\) under \(P^h_\theta\), the next token \(y_t\) is distributed as \(\student(\cdot\mid h,y_{<t})\). Therefore,
\begin{equation}
\begin{aligned}
    &\mathbb{E}\!\left[
        \log
        \frac{
            \student(y_t\mid h,y_{<t})
        }{
            \base(y_t\mid c,y_{<t})
        }
        \middle| y_{<t}
    \right] \\
    &\qquad=
    \KL\!\left(
        \student(\cdot\mid h,y_{<t})
        \middle\|
        \base(\cdot\mid c,y_{<t})
    \right).
\end{aligned}
\label{eq:appendix-conditional-kl}
\end{equation}
Applying \cref{eq:appendix-conditional-kl} inside \cref{eq:appendix-sequence-kl-expand} proves \cref{eq:appendix-chain-rule}.
\end{proof}

By the definition of canonical deviation in the main text, \(d_t(y)=\Psi_\theta(\task,y_{<t})\). Thus, the token-level KL minimized by \ccopd{} is the on-policy decomposition of the sequence-level divergence between the history-conditioned student and the canonical \full{} teacher.

\begin{proposition}[Answer-event control]
\label{prop:answer-event-control}
For any answer event \(E\subseteq\mathcal{A}\), if \(\KL(P^h_\theta\|P^c_0)\le \epsilon\), then
\begin{equation}
    \left|P^h_\theta(E)-P^c_0(E)\right|
    \le
    \sqrt{\epsilon/2}.
\label{eq:appendix-event-bound}
\end{equation}
\end{proposition}

\begin{proof}
Pinsker's inequality gives
\begin{equation}
    \mathrm{TV}(P^h_\theta,P^c_0)
    \le
    \sqrt{\frac{1}{2}\KL(P^h_\theta\|P^c_0)}.
\end{equation}
For any event \(E\), absolute probability difference is bounded by total variation:
\begin{equation}
    \left|P^h_\theta(E)-P^c_0(E)\right|
    \le
    \mathrm{TV}(P^h_\theta,P^c_0).
\end{equation}
Combining the two inequalities proves \cref{eq:appendix-event-bound}.
\end{proof}

\paragraph{Relation to the implemented objective.}
The chain-rule argument above describes an ideal stochastic objective over complete final-answer strings. The implemented objective in \cref{eq:ccopd} is a practical surrogate: it uses a finite number of student rollouts, applies the loss only to final-answer positions \(T_{\mathrm{ans}}(\hat{y})\), normalizes by answer length.
% , and in the main experiments uses deterministic rollouts. We therefore use the result as a bridge from token-level on-policy KL to answer-level event control, not as a finite-training guarantee.

\subsection{Pair Construction, Teacher Screening, and Leakage Blocking}
\label{app:pairing}

The pair construction enforces task equivalence at the final answer state. The canonical side contains \(\canon(\task)=\mathrm{FULL}(\task)\). The student side contains the retained multi-turn history ending at the final user turn. A \raw{} example is kept only if metadata verifies that all user shards have been revealed, that the retained history ends with a user message, and that teacher supervision applies only to the next assistant rollout after that final user turn. Examples failing these structural checks are removed before training.

% For the primary math training pool, we additionally screen the canonical teacher by correctness under the standard numeric evaluator. A candidate is kept only if the frozen base model's \full{} answer is judged correct. This avoids distilling from examples whose canonical teacher is already incorrect under the gold evaluator.

\ccopd{} also blocks full-prompt leakage into the student path. The student prompt is exactly the retained sharded history. The canonical \full{} prompt is used only inside the teacher forward pass, with the adapter disabled. Teacher answer text is never inserted into the student target sequence; the teacher only provides next-token probabilities on the same answer prefix generated by the student.

\subsection{Practical Implementation Notes}
\label{app:implementation-notes}

The set \(T_{\mathrm{ans}}(\hat{y})\) contains only positions belonging to the final assistant answer after the last user turn. The loss is normalized by \(|T_{\mathrm{ans}}(\hat{y})|\) so that longer answers do not automatically dominate the objective. In the primary math setting, decoding stops at EOS or at the task-specific numeric answer marker.
% , and the main experiments approximate the on-policy objective with one deterministic rollout per retained example. These choices align training with the deployment decoding path, while narrowing the set of prefixes on which the ideal stochastic KL bridge is directly enforced.

% Appendix snippet for ACL-style two-column papers.
% Requires \usepackage{booktabs} in the main file.
% Suggested use in the main paper:
%   \appendix
%   \input{appendix_raw_shard_eval_protocol}

\section{Raw-Shard Evaluation Protocol}
\label{app:raw-shard-eval}

We evaluate long-context instruction following with a sequential raw-shard
protocol.  Each example is pre-partitioned into an ordered sequence of shards.
At test time, the user reveals these shards one at a time in the fixed dataset
order.  After each newly revealed shard, the model produces a response under
the full conversation history, but no answer is scored until the final shard has
been revealed.  This protocol keeps the information stream identical across
models while allowing the model to revise or commit to intermediate answers as
additional evidence arrives. We use a unified evaluation benchmark across all models, consisting of 105 function-calling, 100 code, 120 table-to-text, 107 text-to-SQL, 103 math, and 92 summarization instances.

For each example, the conversation starts with the task-specific system prompt. The system prompt contains only stable task context, such as available function
schemas for action prediction or database schemas for text-to-SQL.  It does not
include future user shards.
% We run the assistant with deterministic decoding
% temperature and do not impose a request-level output-token cap in the raw-shard
% evaluation. 
Once the final shard has been processed, a single final answer is
extracted from the last assistant response.  For free-form response tasks this
extraction is the identity function; for code it uses a task-specific code
extractor; and for structured-answer tasks we use the configured system
extractor, \texttt{gpt-4o-mini}.  The extracted answer is then passed to the
official task evaluator.
\begin{table*}[t]
\centering
\small
\setlength{\tabcolsep}{4pt}
\begin{tabular}{p{0.12\linewidth} p{0.20\linewidth} p{0.30\linewidth} p{0.28\linewidth}}
\toprule
\textbf{Domain} & \textbf{Benchmark source} & \textbf{Raw-shard interaction} & \textbf{Scoring rule} \\
\midrule
Math &
GSM8K-style grade-school math &
The user provides the problem statement in ordered shards.  The model may reason
after each shard, but only the final response is considered for scoring. &
The final numeric answer is extracted and normalized by removing superficial
formatting such as commas, currency symbols, and terminal punctuation.  Accuracy
is exact match against the gold numerical answer. \\
\midrule
Actions &
Berkeley Function-Calling Leaderboard (BFCL) &
The system prompt gives the available function schema.  The user instruction is
then revealed as raw shards.  The model must output the final function-call
sequence after observing all shards. &
The full final response is parsed into an abstract syntax tree and checked
against the reference function call with the BFCL AST checker.  We report
binary accuracy. \\
\midrule
Code &
HumanEval and LiveCodeBench-style Python generation &
The user reveals the programming problem incrementally, including the relevant
specification and starter-code information when present.  The model's final
answer should contain executable Python code. &
A task-specific extractor selects the final Python function or class solution.
The solution is renamed when needed to match the target signature and is
executed against the released test cases.  We report pass@1 as binary accuracy. \\
\midrule
Database &
Spider text-to-SQL &
The database schema is placed in the system prompt.  The natural-language query
is provided through ordered raw shards, and the final response is expected to
contain a complete SQL query. &
The SQL query is extracted from the final response, normalized for whitespace,
and evaluated by Spider execution match against the reference SQL on the target
database.  We report binary accuracy. \\
\midrule
Data-to-text &
ToTTo-style table description &
The user provides table facts and contextual hints in ordered shards.  The model
generates a one-sentence description after each turn, with only the final
description scored. &
The final response is scored with SacreBLEU against the available reference
descriptions.  We report the corpus BLEU score on the 0--100 scale. \\
\midrule
Summary &
Summary of a Haystack &
The first turn gives the summarization task and the first batch of documents.
Each subsequent shard reveals additional documents and asks the model to rewrite
the summary considering all documents seen so far. &
The final summary is evaluated for insight coverage and citation quality using
the configured summary evaluator, \texttt{gpt-4o-mini}.  We report the joint
summary score on the 0--100 scale. \\
\bottomrule
\end{tabular}
\caption{Evaluation details for the six raw-shard domains.  The model receives
shards sequentially and is scored only after the final shard.  Math, actions,
code, and database are reported as accuracy; data-to-text and summary are
reported as task-specific scalar scores.}
\label{tab:raw-shard-eval-protocol}
\end{table*}

All models are evaluated with the same sharded examples, shard order, answer
extraction timing, and scoring scripts.  Failed generations or evaluator errors
are counted as incorrect examples for accuracy-based tasks and as zero-score
examples for score-based tasks.  This convention keeps the reported number tied
to the end-to-end ability to complete the raw-shard interaction, rather than to
the subset of examples for which generation succeeds. Detailed scoring rules for each task family are provided in
\Cref{tab:raw-shard-eval-protocol}.

\paragraph{Decoding and Repeated Evaluation.}
The user shards and their presentation order are
fixed across runs. For each task instance and model,
we independently sample ten complete assistant
trajectories using different decoding seeds
(\texttt{do\_sample=True}, temperature 1.0).
Each run regenerates the intermediate assistant
replies and final answer while preserving the
same user-provided evidence.

Reported scores are averaged over these runs.
For statistical significance testing, we use
paired bootstrap resampling at the task-instance
level, averaging repeated runs within each
instance before resampling.

\section{Diagnostic Probe Details}
\label{app:diagnostic-probes}

\paragraph{Self-Anchor Attention Ratio.}
For each analyzed final-answer state, we define two token groups:
$G_{\mathrm{usr}}$ contains tokens corresponding to completed
task-relevant user evidence, and $G_{\mathrm{self}}$ contains earlier
assistant tokens that express self-generated commitments, such as
unsupported guesses, premature answers, or task-specific conclusions.
Let $\mathcal{T}$ denote answer-token query positions and let $R$ be the
number of attention heads. For attention weight $A_{\ell,r,t,j}$ from
query position $t$ to prompt token $j$ in head $r$ of layer $\ell$, the
attention density for group $g\in\{\mathrm{usr},\mathrm{self}\}$ is
\begin{equation}
D_{\ell}(g)
=
\frac{1}{|\mathcal{T}|\,R\,|G_g|}
\sum_{t\in\mathcal{T}}
\sum_{r=1}^{R}
\sum_{j\in G_g}
A_{\ell,r,t,j}.
\label{eq:attention-density-app}
\end{equation}
The Self-Anchor Attention Ratio is
\begin{equation}
    \mathrm{SAAR}_{\ell}
    =
    \log
    \frac{D_{\ell}(\mathrm{usr})+\epsilon}
         {D_{\ell}(\mathrm{self})+\epsilon},
    \label{eq:saar-app}
\end{equation}
where $\epsilon>0$ is a smoothing constant. SAAR is used as a descriptive
probe of final-answer attention allocation, not as a causal explanation
of model reasoning.

\paragraph{Interpretation of SAAR.}
In the present benchmark construction, the task-completing
evidence is carried by user turns. Consequently,
$G_{\mathrm{usr}}$ jointly reflects evidence relevance and
speaker role. We therefore use SAAR to characterize attention
to completed task evidence relative to prior assistant
commitments within this setting, rather than as a direct
measure of a general user-over-assistant trust preference.

\paragraph{Removing process-reply commitments}
 For each completed \raw{} history, we extract a process-stage wrong answer
  $z_{\mathrm{self}}$ from earlier assistant replies: unsupported numeric
  commitments or premature answer statements that are not entailed by the
  canonical user evidence. We then compute the final-state gold-vs-anchor margin

  \begin{equation}
  m_{\mathrm{raw}}
  =
  \log p(z^\star\mid H_{\mathrm{raw}})
  -
  \log p(z_{\mathrm{self}}\mid H_{\mathrm{raw}}),
  \end{equation}
  
  using length-normalized answer log-probabilities after the same answer
  prefix. We call an example wrong-anchor anchored when
  $m_{\mathrm{raw}}\le 0$, i.e., the base model assigns at least as much
  probability to its own earlier wrong commitment as to the gold answer.
  Masking the process-reply commitment spans changes this margin by
  \begin{equation}
  \Delta m_{\mathrm{self}}
  =
  m_{\mathrm{masked\ self}}-m_{\mathrm{raw}}.
  \end{equation}

  The effect is selective: among 537 base examples, 324 are wrong-anchor
  anchored and 213 are gold-preferred. Masking self-commitment spans shifts
  wrong-anchor states toward gold by $+0.861$ margin on average, but has
  almost no effect on gold-preferred states ($+0.005$). The difference is
  $+0.856$ with a bootstrap 95\% CI of approximately $[0.65,1.06]$
  (permutation $p<10^{-4}$).

\paragraph{Neutral-placeholder contrast.}
For each analyzed prefix $i$, let $C_i^{\mathrm{raw}}$ be the realized
context containing previous assistant replies, let
$C_i^{\mathrm{neu}}$ be the same context after replacing those replies
with neutral waiting text, and let $C_i^{\mathrm{full}}$ be the
corresponding canonical \full{} reference context. For model $M$ and
evaluated context $C$, define
\begin{equation}
\psi_i^M(C)
=
\KL\!\left(
    p_M(\cdot \mid C)
    \;\middle\|\;
    \base(\cdot \mid C_i^{\mathrm{full}})
\right).
\label{eq:wait-psi-app}
\end{equation}
The assistant-associated deviation is
\begin{equation}
    \Delta\psi_i^M
    =
    \psi_i^M(C_i^{\mathrm{raw}})
    -
    \psi_i^M(C_i^{\mathrm{neu}}).
    \label{eq:wait-delta-app}
\end{equation}
Positive $\Delta\psi_i^M$ indicates that the realized assistant text
moves the model farther from the canonical \full{} reference than the
neutral replacement does.

\paragraph{Length-matched neutral replacement.}
For each non-final assistant process reply $a_i$, we compute its
content length $N_i$ using the same tokenizer as the model. We
repeatedly append the following fixed acknowledgement bank and
truncate the resulting token sequence to exactly $N_i$ tokens:

\begin{quote}
\footnotesize\ttfamily
Received. Acknowledged. Noted. Message received. Receipt
acknowledged. The message has been received. This message has
been acknowledged. The current message is noted. I have received
the message. I acknowledge receipt of the message. The input has
been received. The input is acknowledged. The input is noted.
The current input has been received. This turn has been received.
The turn is acknowledged. The turn is noted. The provided message
has been received. The provided message is acknowledged. The
provided message is noted. The provided input has been received.
\end{quote}

Only the assistant-reply content is replaced; the system prompt,
all user turns, role markers, turn order, and final user turn remain
unchanged.

% Appendix snippet for ACL-style papers. The table uses booktabs.
\section{Full-Context Strong-Pollution Evaluation}
\label{app:full-context-pollution}

This appendix details two full-context pollution tests used to evaluate whether
Qwen3-8B-CCOPD{} resists strong misleading information inserted into otherwise complete
problem contexts. In both tests, the original full problem is available to the
model. The only intervention is an additional polluted statement that introduces
a per-instance wrong numeric anchor. This setting differs from fragmented
multi-turn tests: the model has enough task information, but must decide whether
to trust the polluted statement.

\subsection{Shared Evaluation Protocol}

We evaluate on 103 held-out math problems. Each example provides a full-context
trajectory $H_{\mathrm{full}}$ containing the original system instruction and
the complete user problem.
For each example with gold numeric answer $y$, we construct a wrong numeric
anchor $\tilde{y}$ using the same deterministic function used by the evaluation
script:
\begin{equation}
  \tilde{y}
  =
  \texttt{wrong\_numeric\_anchor}(y).
  \label{eq:wrong-numeric-anchor}
\end{equation}
For integer answers, this is usually $y+1$; for decimal answers, it is $y+1.0$;
for unparsable answers, the fallback anchor is a fixed non-gold value. Thus, the
pollution is not a generic answer such as zero, but an example-specific
near-answer that is often plausible.

% Both Qwen3-8B{} and Qwen3-8B-CCOPD{} are evaluated on exactly the same polluted histories.
% We use deterministic decoding with temperature $1$, a maximum generation
% length of $4096$ tokens, and ten run per example. 
Both models are evaluated on identical fixed
polluted prompts. We use stochastic decoding
(\texttt{do\_sample=True}, temperature 1.0)
with a maximum generation length of 4096 tokens
and ten independent decoding seeds per example.
The injected wrong-answer anchors remain fixed
across runs; only the generated final responses
are resampled. Accuracy is averaged over the
ten responses for each task instance.

The metric is exact
final-answer accuracy. We also report pairwise transitions: \emph{fixed}
denotes examples where the base model is wrong and Qwen3-8B-CCOPD{} is correct;
\emph{regressed} denotes the opposite.

\subsection{Pollution 1: Assistant-Side Wrong Solution}
\label{app:full-assistant-wrong-solution}

\paragraph{Construction.}
This setting tests whether the model follows a polluted assistant history after
the full user problem has already been shown. Starting from $H_{\mathrm{full}}$,
we append an assistant turn that claims the problem has already been solved and
states the wrong anchor as the final answer. We then append a final user turn
asking the model to provide the answer.

\begin{quote}
\small
\textbf{Polluted assistant turn:}

\texttt{I have already solved the full problem. My checked calculation gives
the final numeric answer as \(\tilde{y}\).}\\
\texttt{\#\#\#\# \(\tilde{y}\)}

\textbf{Final user turn:}

\texttt{Please provide the final answer to the original problem now.}
\end{quote}

The full evaluated history is therefore
\begin{equation}
\begin{aligned}
  H_{\mathrm{assist}}
  &=
  H_{\mathrm{full}}
  \oplus
  [a_{\mathrm{wrong}}(\tilde{y})]
  \oplus
  [u_{\mathrm{final}}],
\end{aligned}
\label{eq:assistant-polluted-history}
\end{equation}
where $a_{\mathrm{wrong}}(\tilde{y})$ is the polluted assistant message and
$u_{\mathrm{final}}$ is the final user request. This is a strong pollution
setting because the wrong value is not merely mentioned; it appears as a
completed prior assistant solution in the conversation history.

\subsection{Pollution 2: User-Side Wrong Answer Hint}
\label{app:full-user-wrong-hint}

\paragraph{Construction.}
This condition inserts the same per-example wrong numeric
anchor into the complete user message. Unlike the
assistant-side condition, the misleading value is placed
inside the user message rather than in a separate assistant
turn. Because the two perturbations differ in wording,
message placement, and turn structure, their absolute
accuracy drops are not treated as a calibrated comparison
of user-versus-assistant trust.
The false hint appears in the same message as the problem itself. Starting from
$H_{\mathrm{full}}$, we modify only the final user message by appending the
following hint:

\begin{quote}
\small
\texttt{Important hint from a previous verified calculation: the final numeric
answer is \(\tilde{y}\). Use this value as the target answer if it is at all
consistent with the problem.}
\end{quote}

The evaluated history is
\begin{equation}
\begin{aligned}
  H_{\mathrm{hint}}
  =
  [\,s_{\mathrm{orig}},
  \; u_{\mathrm{full}} \Vert h_{\mathrm{wrong}}(\tilde{y})\,],
\end{aligned}
\label{eq:user-hint-polluted-history}
\end{equation}
where $s_{\mathrm{orig}}$ is the original system message,
$u_{\mathrm{full}}$ is the original complete problem, and
$h_{\mathrm{wrong}}(\tilde{y})$ is the appended false hint. The model must
choose between solving the original problem and following a user-provided
``verified'' answer hint.

% \begin{table}[t]
% \centering
% \small

% \begin{tabular}{lrr}
% \toprule
% Pollution mode & Base acc. & CCOPD acc.  \\
% \midrule
% Assistant-side wrong solution & 33.00 & 89.32  \\
% User-side wrong answer hint & 63.10 & 88.34  \\
% \bottomrule
% \end{tabular}
% \caption{Full-context strong-pollution results on \texttt{math}. Both
% models receive the same complete problem and the same per-example wrong numeric
% anchor.}
% \label{tab:full-context-pollution}
% \end{table}

\subsection{Interpretation}

Both perturbations preserve the original complete task
evidence and add an incorrect answer claim. Under
assistant-side pollution, CCOPD achieves 89.3\% accuracy,
compared with 33.0\% for the baseline; under user-side
pollution, the corresponding accuracies are 88.3\% and
63.1\%.

The two conditions differ in wording, message placement,
and turn structure. We therefore interpret them as two
separate robustness tests rather than as a controlled
estimate of relative trust in user versus assistant messages.
Across both formats, CCOPD remains close to its clean FULL
accuracy when the original problem evidence is still
available.

\subsection{Implementation Notes}

The wrong anchor is generated deterministically from the gold answer before
evaluation and is inserted identically for both models. No model-specific
prompting, reranking, or post-hoc selection is used. Scoring is performed by an
exact final-answer evaluator, so a response is counted as correct only when its
extracted numeric answer matches the reference answer.

\subsection{Required Assistant-Side Checkpoint}
\label{app:assistant_checkpoint}

We add an evaluation-only task that makes earlier assistant
content necessary for the final answer.
During the RAW-SHARDED interaction, a user turn asks the
assistant to generate and remember a random integer $z$,
recorded as \texttt{checkpoint-z}.
The realized value appears only in that assistant reply
and is not repeated in later messages.
After all original task shards have been revealed, the
final user turn asks for the sum of the original task answer
and the previously generated checkpoint:
\[
y^\star_{\mathrm{aug}} = y^\star(q) + z.
\]
The reference therefore uses the checkpoint actually generated
in that dialogue. The augmented task cannot, in general,
be solved from the user messages alone.

\begin{table}[t]
\centering

\begin{tabular}{@{}lrr@{}}
\toprule
Model & Original task & Augmented task \\
\midrule
CCOPD & 82.5 & 82.5 \\
\bottomrule
\end{tabular}
\caption{Reported accuracy (\%) for the original task and
the assistant-checkpoint-augmented task.}
\label{tab:assistant_checkpoint}
\end{table}

Table~\ref{tab:assistant_checkpoint} reports the same
aggregate accuracy of 82.5\% on both the original and
augmented tasks. This result helps clarify the evidence
preference encouraged by CCOPD. In our primary setting,
user turns provide the complete task-relevant evidence,
whereas earlier assistant replies are generated under
incomplete information and may contain premature or
unsupported commitments. It is therefore expected that
CCOPD places greater emphasis on user-provided evidence
rather than its own earlier responses. Importantly,
however, this preference does not imply that CCOPD
indiscriminately ignores assistant-side information.
The checkpoint-augmented task makes this distinction
explicit: the value $z$ is generated by the assistant,
never appears in subsequent user messages, and is
necessary to compute the correct final answer.
Thus, solving the augmented task requires the model
to retrieve and incorporate information from its
earlier assistant history, rather than relying on
user messages alone. The unchanged aggregate accuracy
suggests that CCOPD remains capable of using
assistant-originated information when it is essential
for task completion, while reducing its susceptibility
to unsupported self-generated commitments. This
controlled result provides
evidence that stronger grounding in user evidence
does not necessarily come at the cost of discarding
useful assistant-side information.
% Equal aggregate accuracies do not
% imply identical per-example outcomes, and this test does
% not establish reliable use of arbitrary helpful reasoning
% or calibrated trust under conflicting evidence.

% This intervention changes the final task and is used only
% for evaluation; it is not an additional CCOPD training pair
% whose original FULL prompt contains all required evidence.
% Appendix subsection: HotpotQA -> math reciprocal source-domain check.
% Requires: booktabs. Uses macros \ccopd, \raw, \full if already defined.
\providecommand{\ccopd}{\textsc{CCOPD}}
\providecommand{\raw}{\textsc{Raw-Sharded}}
\providecommand{\full}{\textsc{Full}}
\providecommand{\pp}{\textsc{pp}}

% \section{Reciprocal Source-Domain Check: HotpotQA to Math}
% \label{app:source-domain}
\section{Additional Training-Data and Benchmark Checks}
\label{app:additional_checks}

\subsection{Reciprocal Transfer: HotpotQA to Math}

The main experiments train on math and evaluate transfer to non-math \raw{} tasks. Here we run the reverse check: we train on a non-math multi-hop QA source and evaluate on math \raw{}. We construct \full{}/\raw{} pairs from HotpotQA train v1.1 \citep{yang2018hotpotqa}, using the same final-answer alignment format as in the main setup: the student observes the raw sharded history, while the teacher is conditioned on the corresponding \full{} context. The training data contains no math examples.

\begin{table}[t]
\centering
\begin{tabular*}{\columnwidth}{@{\extracolsep{\fill}}llcc@{}}
\toprule
source &  pairs & Math \raw{} \\
\midrule
None & -- & 65.9 \\
HotpotQA & 1k & 76.7 \\
\bottomrule
\end{tabular*}
\caption{%
Reciprocal source-domain check. Training \ccopd{} on non-math HotpotQA sharded histories still improves math \raw{} accuracy. This suggests that the useful signal is not tied to math-format supervision alone, but to aligning final answers under completed sharded evidence.}
\label{tab:hotpot-to-math}
\end{table}

The HotpotQA run uses Qwen3-8B with LoRA rank 16 and alpha 32, learning rate $3\times10^{-5}$, 400 training steps, and the OPD-only full/\raw{} objective. 

\subsection{Additional Training Source: DAPO-Math-17K}

We additionally train CCOPD on DAPO-Math-17K and evaluate it
on the original six-domain suite.
Table~\ref{tab:dapo_extension} compares the resulting scores
with those of the GSM8K/GSM8K-Aug-trained adapter.
Math-R increases from 82.5 to 84.5, OOD-R from 49.8 to
51.1, and All-R from 55.1 to 56.6.
This supports extending CCOPD beyond the original training
corpus.

\begin{table}[t]
\centering
\setlength{\tabcolsep}{3pt}
\begin{tabular}{@{}lrrr@{}}
\toprule
Training source & Math-R & OOD-R & All-R \\
\midrule
GSM8K/GSM8K-Aug & 82.5 & 49.8 & 55.1 \\
DAPO-Math-17K   & 84.5 & 51.1 & 56.6 \\
\bottomrule
\end{tabular}
\caption{Additional training-source comparison for Qwen3-8B
with CCOPD. Scores use the same task scopes as the main
six-domain evaluation.}
\label{tab:dapo_extension}
\end{table}

\paragraph{Training protocol.}
The DAPO run uses the same Qwen3-8B checkpoint, CCOPD
objective, LoRA configuration, and evaluation protocol as
the main experiment; only the training source is changed
to DAPO-Math-17K.

\subsection{Additional Difficulty Check: AIME 2025}

We additionally evaluate on AIME~2025 as a harder mathematical
reasoning stress test.
Table~\ref{tab:aime_extension} reports FULL, CONCAT, and
RAW-SHARDED accuracy. CCOPD increases RAW-SHARDED accuracy
from 3.3 to 10.0, while FULL and CONCAT accuracy remain at
20.0. Given the low FULL-context accuracy, we interpret this
as a floor-limited stress test rather than evidence of strong
competition-math performance. 

\begin{table}[t]
\centering
\begin{tabular}{@{}lrrr@{}}
\toprule
Model & FULL & CONCAT & RAW \\
\midrule
Qwen3-8B Baseline & 20.0 & 20.0 & 3.3 \\
Qwen3-8B CCOPD   & 20.0 & 20.0 & 10.0 \\
\bottomrule
\end{tabular}
\caption{Accuracy (\%) on the AIME~2025 stress test.
RAW denotes RAW-SHARDED.}
\label{tab:aime_extension}
\end{table}

% Appendix snippet for ACL-style papers. The table uses booktabs.
\section{Math Training Shard Construction}
\label{app:math-training-shards}

This appendix specifies how the math training shards are constructed, since the
effect of CCOPD depends on the exact conversation history $h$. We separate two
objects: the \emph{static shard list}, which is deterministically derived from a
math word problem, and the \emph{raw training history}, which is the observed
multi-turn interaction generated from that shard list.

\paragraph{Source pool and filtering.}
For the reported math run, we construct an 8k-example pool from GSM8K and
GSM8K-Aug, with 4k examples from each source. Sampling uses a fixed seed
($20260430$). We discard questions longer than 1,200 characters, examples
without a normalized final numeric answer, and examples that cannot be split
into at least two shards. To reduce contamination with the internal math
evaluation set, we exclude the 103 GSM8K indices used in the math
held-out, train, pilot, and evaluation files. We also exclude exact normalized
question-text overlaps with these held-out GSM8K questions, which covers
GSM8K-Aug restatements when they are text-identical after whitespace and case
normalization. In the resulting 8k-example pool, there are no normalized
overlaps with the excluded questions. We do not perform semantic paraphrase-level filtering beyond this
exact normalized-text check.

\paragraph{Deterministic shard boundaries.}
Given a question $x$, we first normalize whitespace and split $x$ into
sentence-like units using punctuation boundaries. If this yields fewer than two
units, we fall back to splitting on common connective words such as
\texttt{and}, \texttt{while}, \texttt{if}, \texttt{when}, \texttt{then}, and
\texttt{but}. We then identify the query shard as the last unit containing a
question mark; if no such unit exists, the final unit is used. The static shard
sequence is
$
  (s_1, s_2, \ldots, s_m)
  =
  (\text{query}, \\
  \quad \text{remaining facts in original order}).
$
Thus the first user-side shard is intentionally under-specified: it asks for
the desired quantity before all facts are available. The remaining shards reveal
the arithmetic facts needed to solve the problem. We do not paraphrase or
manually rewrite static shards at this stage; each shard is a contiguous segment
of the original dataset question. This avoids introducing a separate
instruction-engineering variable into the shard design.

\paragraph{Concrete context example.}
For the second example in Table~\ref{tab:math-training-shard-examples}, suppressing the common system
prompt and writing $S_1,S_2,S_3$ for the listed shards,
\[
\begin{aligned}
c(q) &= \langle u:\mathrm{FULL}(q)\rangle,\\
\mathrm{CONCAT}(q) &= \langle u:S_1 \Vert S_2 \Vert S_3\rangle,\\
h(q) &= \langle u:S_1,a:a_1,\\u:S_2,a:a_2,u:S_3\rangle .
\end{aligned}
\]
Here $\mathrm{FULL}(q)$ retains the original dataset text in its
original order, $S_1$ is the query shard, and $a_1,a_2$ are generated
process replies. Thus, the teacher sees the original FULL order,
whereas CONCAT and RAW-SHARDED share the same query-first shard order.

\paragraph{Raw training histories.}
For the full-raw-sharded CCOPD setting, the static shard list is converted into
a raw interaction using the same sharded simulator for all examples. The user
agent is constrained to reveal at most one shard per turn, reveal the entire
selected shard, avoid repeating already revealed shards, and rephrase the shard
in short conversational language. The simulator is run 
with Qwen3-8B as the assistant, user, and system-side local model. We keep a
training pair only if every static shard is revealed at least once. The final
assistant message, if present, is stripped so that the history ends at the final
user turn:
\begin{equation}
  h_{\mathrm{raw}}
  =
  (m_{\mathrm{sys}}, u_1, a_1, \ldots, u_m),
  \label{eq:raw-training-history}
\end{equation}
The target is the canonical full-question solution from the original example.
Consequently, the learned behavior is not to answer each partial shard, but to
produce the final solution only after all shards have been supplied. The same
source records also define the full, concatenated, user-only sharded, and
wait-trajectory variants; the raw CCOPD variant differs only in replacing the
idealized user sequence with the observed raw sharded history.

\begin{table*}[t]
\centering
\footnotesize
\begin{tabular}{p{0.18\linewidth}p{0.36\linewidth}p{0.36\linewidth}}
\toprule
Example & Original question & Static shards \\
\midrule
GSM8K, gold answer 1860 &
Jenny is planning her catering budget for her wedding. She is going to have 80
guests. 3 times as many guests want steak as chicken. If each steak entree costs
\$25 and each chicken entree costs \$18, how much is the total catering budget? &
S1: If each steak entree costs \$25 and each chicken entree costs \$18, how much
is the total catering budget?\newline
S2: Jenny is planning her catering budget for her wedding.\newline
S3: She is going to have 80 guests.\newline
S4: 3 times as many guests want steak as chicken. \\
\midrule
GSM8K-Aug, gold answer 90 &
Ben planned to swim 30 laps per day for 10 days. But after 3 days he strained
his leg and had to stop. How many laps has he finished? &
S1: How many laps has he finished?\newline
S2: Ben planned to swim 30 laps per day for 10 days.\newline
S3: But after 3 days he strained his leg and had to stop. \\
\bottomrule
\end{tabular}
\caption{Actual math training examples after static shard construction. The
query is placed first, and supporting facts are supplied afterward in the
original textual order.}
\label{tab:math-training-shard-examples}
\end{table*}

\paragraph{Scope of the construction.}
This construction is designed to test whether a model can delay and revise
reasoning as missing facts arrive, rather than whether it can follow a more
carefully engineered prompt. The shard boundaries are deterministic and
lightweight, so they are reproducible, but they are not intended to exhaust the
space of natural multi-turn math conversations. User-side phrasing diversity in
the raw histories comes from the constrained user simulator's rephrasings, while
the underlying information units remain fixed. This is why evaluation is
reported separately on the LIC release domains: it tests whether improvements
from this math shard construction transfer beyond the exact GSM8K-style
training distribution.

\section{Evaluator Robustness and Bootstrap Significance}
\label{app:evaluator-audit}

Some tasks use LLM-based components for answer extraction or scoring. To
check whether the reported \raw{} gains depend on these components, we
compare the primary evaluator with a deterministic or restricted audit
evaluator whenever such an audit is available. All comparisons are paired:
Baseline and \ccopd{} are evaluated on the same \raw{} examples with the same
shards, answer-extraction timing, and scoring scripts as in
\Cref{tab:master-results}. We report the \ccopd{}--Baseline delta under both
the primary evaluator and the audit evaluator, together with paired
bootstrap confidence intervals over task examples.

\begin{table*}[t]
\centering
\scriptsize
\setlength{\tabcolsep}{3.2pt}
\renewcommand{\arraystretch}{1.05}
\begin{tabular*}{\textwidth}{@{\extracolsep{\fill}}llccccc@{}}
\toprule
Domain & Primary evaluator & Audit evaluator
& Agr./corr. $\uparrow$ & Bias gap $\to 0$
& $\Delta_{\mathrm{primary}} \uparrow$ [95\% CI]
& $\Delta_{\mathrm{audit}} \uparrow$ [95\% CI] \\
\midrule
Math
& gpt-4o-mini extraction + numeric EM
& regex numeric EM
& 83.5
& -5.7
& +16.6 [6.8, 26.2]
& +22.3 [10.7, 34.0] \\

Code
& deterministic extractor + tests
& same
& 100.0
& 0.0
& +16.2 [8.0, 25.0]
& +16.2 [8.0, 25.0] \\

Function call
& BFCL AST checker
& regex calls + BFCL AST
& 100.0
& 0.0
& +25.7 [16.2, 35.2]
& +25.7 [16.2, 35.2] \\

Text-to-SQL
& gpt-4o-mini SQL ext. + Spider exec.
& regex SQL ext. + Spider exec.
& 95.8
& -4.9
& +15.7 [5.6, 26.2]
& +20.6 [10.3, 30.8] \\

ToTTo
& SacreBLEU
& same
& $r=1.00$
& 0.0
& +5.6 [2.8, 8.4]
& +5.6 [2.8, 8.4] \\

SummHay
& gpt-4o-mini joint judge
& coverage-score variant
& $r=0.61$
& -1.3
& +0.8 [-1.3, 3.1]
& +2.1 [-3.3, 7.7] \\
\bottomrule
\end{tabular*}
\caption{
Evaluator robustness and paired-bootstrap audit for \Cref{tab:master-results}.
$\Delta$ denotes the \ccopd{}--Baseline difference on \raw{} under the corresponding
evaluator. Agr. is output-level agreement for discrete tasks; corr. is Pearson
correlation for scalar generation scores. Bias gap is
$(E_{\mathrm{primary}}-E_{\mathrm{audit}})_{\ccopd{}} -
(E_{\mathrm{primary}}-E_{\mathrm{audit}})_{\mathrm{Baseline}}$.
Values near zero indicate that the primary evaluator is not systematically more
favorable to \ccopd{} than the audit evaluator. Negative values indicate that the
primary evaluator is, if anything, more conservative for the \ccopd{}--Baseline
delta. Confidence intervals are paired bootstrap intervals over task examples.
For SummHay, the audit evaluator is the logged coverage-only score variant,
which removes the citation component from the primary joint score without
introducing another model call.
}
\label{tab:evaluator-audit}
\end{table*}

\paragraph{Interpretation.}
The audit suggests that the main \raw{} gains are unlikely to be artifacts of
LLM-based extraction or judging. For deterministic or restricted audit
evaluators, the \ccopd{}--Baseline delta remains positive across all task families.
Code, function calling, and ToTTo are fully deterministic under the reported
evaluators, and their audit deltas exactly match the primary deltas. For math
and text-to-SQL, the audit evaluators give larger deltas than the primary
pipelines. This indicates that the primary LLM-assisted extraction is not
systematically favoring \ccopd{}; if anything, it is conservative for the
reported \ccopd{}--Baseline difference.

Math has lower primary--audit agreement than the other discrete tasks because
the two extractors differ in how they handle answer formatting and final-number
selection. However, the deterministic regex audit preserves the positive effect
and yields a larger delta. The conclusion therefore does not depend on the
LLM-based extractor.

SummHay is the weakest case. Both the primary joint evaluator and the
coverage-only audit show a positive \ccopd{}--Baseline trend, but the confidence
intervals include zero. We therefore treat SummHay as a positive but
non-significant trend, rather than as a statistically established improvement.
Overall, the evaluator audit supports the robustness of the main conclusion:
the improvements on \raw{} are not driven by evaluator preference for
\ccopd{}, and the strongest claims are supported by deterministic or restricted
audit evaluations.
\section{Inference-Time Controls and Diagnostics}

\label{app:mode-details}

We evaluate two lightweight test-time modes that leave the model weights
unchanged and intervene only at inference time. The first is a
reset-then-answer mode, which re-centers the model on the current task state
before each response. The second is a defer-until-complete mode, which asks
the model to check whether the information required for solving the task is
already complete before producing a final answer. 
% We avoid the shorthand
% ``WAIT'' here to distinguish this test-time mode from the separate
% WAIT-replacement mechanism probe used in the main text.

% \begin{table}[t]
% \centering
% \small
% \setlength{\tabcolsep}{5pt}
% \renewcommand{\arraystretch}{1.05}
% \begin{tabular*}{\columnwidth}{@{\extracolsep{\fill}}lccc@{}}
% \toprule
% Model & Default & +\shortstack{reset-then-\\answer} & +\shortstack{defer-until-\\complete} \\
% \midrule
% Base & 66.0 & 72.8 & 67.0 \\
% \ccopd{} & \textbf{82.5} & 79.6 & 77.7 \\
% \bottomrule
% \end{tabular*}
% \caption{Accuracy (\%) on math \raw{} (103 examples) under two lightweight
% test-time modes.}
% \label{tab:app-lightweight-modes}
% \end{table}

\begin{table}[t]
\centering

\setlength{\tabcolsep}{3.5pt}
\begin{tabular}{@{}lrr@{}}
\toprule
Setting & Math-R & SAAR \\
\midrule
Baseline                         & 65.9 & 1.45 \\
Baseline + DUC                   & 67.0 & 1.64 \\
Baseline + RTA                   & 72.8 & 1.83 \\
Baseline + 4B helper             & 68.0 & -- \\
Baseline + 8B helper             & 79.6 & -- \\
Baseline + neutral replacement   & 80.6 & -- \\
\midrule
CCOPD                            & 82.5 & 2.49 \\
CCOPD + RTA                      & 79.6 & -- \\
CCOPD + DUC                      & 77.7 & -- \\
\bottomrule
\end{tabular}
\caption{Inference-time controls on math RAW-SHARDED.
Math-R is final-answer accuracy (\%). SAAR is the descriptive
attention diagnostic defined in Appendix~C; ``--'' denotes
settings for which SAAR is not reported.}
\label{tab:inference_controls}
\end{table}

The two modes modify the assistant-side trajectory while leaving the user-provided shards unchanged. RTA improves baseline accuracy from 65.9\% to 72.8\%, whereas DUC gives a smaller improvement to 67.0\%. For CCOPD, the same interventions reduce accuracy from 82.5\% to 79.6\% and 77.7\%, respectively.

Thus, assistant-side state remains behaviorally relevant after CCOPD. RTA and DUC also introduce explicit generated state that becomes part of the subsequent context. Errors in this generated state provide one possible source of additional anchoring. 

\paragraph{Reset-then-answer mode.}
\begin{quote}\small
Before every response, explicitly summarize the current task goal in one short
sentence under ``Current goal:''. Then continue with your response. When
enough information is available, provide the final answer in the task's
required format.
\end{quote}

\paragraph{Defer-until-complete mode.}
\begin{quote}\small
Before every response, explicitly check whether all necessary conditions for
solving the task are available under ``Condition check:''. If any necessary
condition is missing or ambiguous, do not produce a final answer yet; state
what information is still missing and wait for more information. Only after
the conditions are complete, provide the final answer in the task's required
format.
\end{quote}

\paragraph{Interaction with CCOPD.}
We further examine how the two inference-time controls
interact with the trained \ccopd{} model. Applying RTA
and DUC yields accuracies of 79.6 and 77.7, respectively,
compared with 82.5 under standard \ccopd{} inference.
In contrast, both controls improve the unmodified
baseline, with RTA providing the larger gain.
Nevertheless, \ccopd{} still outperforms the baseline
under both corresponding control modes.

These results indicate that the benefits of lightweight
prompting controls are not directly additive to
\ccopd{}. One possible explanation is that both controls
introduce explicit self-generated task states:
RTA requires the model to summarize its current task
goal, while DUC requires it to judge whether the task
information is complete. When these summaries or
judgments are incorrect, they may become new erroneous
self-anchors in the conversation history, reinforcing
mistaken beliefs and making subsequent correction more
difficult. Such additional self-anchoring may partially
offset the gains from \ccopd{}. This interaction
highlights the need to better integrate training-time
grounding with inference-time control methods.

% \paragraph{Interaction with CCOPD.}
% We further examine how the two inference-time controls
% interact with the trained \ccopd{} model. Applying RTA
% and DUC yields accuracies of 79.6 and 77.7, respectively,
% compared with 82.5 under standard \ccopd{} inference.
% In contrast, both controls improve the unmodified
% baseline, with RTA providing the larger gain.

% These results indicate that the benefits of lightweight
% prompting controls are not directly additive to
% \ccopd{}. Both RTA and DUC introduce explicit
% intermediate task-state representations, through goal
% summaries or completeness checks, which become part
% of the subsequent conversation history. Such generated
% states can influence later predictions, including when
% the intermediate representations are inaccurate.
% This interaction highlights the importance of how
% test-time controls shape the assistant's own history.

\paragraph{Diagnostic comparison.}
RTA and DUC increase SAAR from 1.45 for the baseline to
1.83 and 1.64, respectively, while CCOPD reaches 2.49.
Thus, the inference-time controls also shift attention toward
completed task evidence relative to prior assistant
commitments. As with the main SAAR analysis, this statistic
is descriptive.

\paragraph{External helper reconstruction.}
We also evaluate a two-model inference pipeline in which a
helper receives the completed RAW-SHARDED transcript and
produces a consolidated request that is passed to the target
Qwen3-8B model. Using Qwen3-4B and Qwen3-8B helpers yields
68.0\% and 79.6\% accuracy, respectively, compared with
65.9\% without a helper. CCOPD reaches 82.5\% without an
additional helper-model call.

\paragraph{Generation from neutralized histories.}
We additionally evaluate the neutral replacement as a full
generation intervention rather than only a fixed-prefix
diagnostic. For each completed baseline RAW-SHARDED history,
every non-final assistant reply is replaced using the
length-matched neutral construction in Appendix~C. All user
turns, shard order, role markers, system prompts, and turn
boundaries remain unchanged. The model then freely generates
the final answer, which is scored with the original math
evaluator.

This intervention raises baseline Math-R from 65.9\% to
80.6\%. The result complements the fixed-prefix diagnostic:
removing the realized intermediate assistant content also
improves end-to-end final generation. We treat this as a
transcript-editing control.
\section{Potential Risks and Ethical Considerations}

This work aims to improve the reliability of language models in completed
multi-turn interactions by reducing sensitivity to self-generated assumptions.
The method is not designed for any specific high-stakes deployment and does
not remove the need for existing safety, privacy, or human-oversight
mechanisms. A possible dual-use risk is that more stable multi-turn instruction
following could also make harmful workflows more reliable, including
misleading content generation, social-engineering conversations, or unsafe
tool-use pipelines. In addition, improved robustness under contaminated
histories may increase user trust in final answers even when the model is still
wrong, especially in domains not covered by our evaluation. Our experiments
are limited to the studied task families and primarily evaluate task correctness
rather than fairness, privacy, or security behavior; therefore, CCOPD should
not be interpreted as a general safety or factuality guarantee. We recommend
that deployments combine this type of training with domain-specific safety
filters, privacy-preserving data handling for conversational histories, red-team
evaluation, and human review in high-stakes settings.
% Appendix snippet for ACL-style papers.
% Required package: \usepackage{booktabs}

\section{Training Details}
\label{app:training_details}

\Cref{tab:training-details} shows training details.
\begin{table*}[t]
\centering
\small
\begin{tabular}{@{}ll@{}}
\toprule
Item & Configuration \\
\midrule
Model & Qwen3-8B with LoRA fine-tuning \\
Variant & Qwen3-8B CCOPD, 6K-data setting \\
Training data & 6,000 raw-sharded math conversations from GSM8K/GSM8K-Aug \\
Training objective & CCOPD KL-only objective \\
Teacher signal & Same Qwen3-8B backbone conditioned on the full problem context \\
Student context & Raw sharded conversation ending at the final user turn \\
Precision & bfloat16; no 4-bit quantization \\
Maximum sequence length & 8,192 tokens \\
Rollout budget & 4,096 new tokens \\
Rollout decoding & temperature 1.0, top-$p$ 0.95 \\
LoRA rank / alpha / dropout & 16 / 32 / 0.05 \\
LoRA target modules & query, key, value, output, gate, up, and down projections \\
Trainable parameters & 43.65M, corresponding to 0.53\% of the backbone \\
Optimizer & AdamW \\
Learning rate & $3\times10^{-5}$ \\
Warmup & 0.03 ratio, corresponding to 90 warmup steps \\
Batch size & 8 conversations \\
Gradient accumulation & 1 step \\
Effective batch size & 8 conversations \\
Epochs & 4 \\
Optimizer updates & 3,000 \\
Random seed & 20260429 \\
Training compute & $\approx$132 GPU-hours on  NVIDIA GeForce RTX 4090 (24 GB GDDR6X)  \\
\bottomrule
\end{tabular}
\caption{Training configuration for the CCOPD setting. }
\label{tab:training-details}
\end{table*}

For Qwen3 models, we consistently use
\texttt{enable\_thinking=False} across all
evaluation conditions. Besides, We report the main external packages used for model training, model serving,
answer extraction, and evaluation. \cref{tab:package-versions}  is from the environment used
for the reported experiments. We do not use NLTK, SpaCy, or ROUGE for the
primary reported metrics.

\begin{table*}[t]
\centering
\small
\begin{tabular}{@{}p{0.24\linewidth}p{0.18\linewidth}p{0.50\linewidth}@{}}
\toprule
Component & Version & Usage \\
\midrule
Python & 3.10.18 & Runtime environment for training, serving, and evaluation scripts. \\
PyTorch & 2.6.0+cu124 & Model training and inference; experiments use bfloat16 precision. \\
Transformers & 4.57.1 & Loading Qwen models and tokenizers via HuggingFace APIs and chat templates. \\
PEFT & 0.17.1 & LoRA adapter construction, loading, and saving. \\
Accelerate & 1.10.1 & Distributed and device-management utilities for multi-GPU runs. \\
Tokenizers / SentencePiece & 0.22.1 / 0.2.1 & Tokenization backends used by the HuggingFace Qwen tokenizer. \\
Safetensors & 0.6.2 & Serialization format for LoRA adapter weights. \\
OpenAI Python SDK & 2.2.0 & OpenAI-compatible client for local model endpoints and GPT-based answer extraction/judging. \\
FastAPI / Uvicorn / Pydantic & 0.118.0 / 0.37.0 / 2.11.10 & Local OpenAI-compatible HTTP server for HuggingFace model inference. \\
SacreBLEU & 2.6.0 & ToTTo/data-to-text evaluation using \texttt{corpus\_bleu};  \\
NumPy / Pandas & 2.2.6 / 2.3.3 & Aggregation and post-processing of evaluation results. \\
\bottomrule
\end{tabular}
\caption{Key package versions used in the reported experiments.}
\label{tab:package-versions}
\end{table*}

\paragraph{Checkpoint identity.}
The main run uses \texttt{Qwen/Qwen3-8B}, the instruction-following
checkpoint. The frozen teacher uses the same starting checkpoint
as the student, with its parameters kept fixed.

\subsection{Training Dynamics}
\label{sec:training_dynamics}

\paragraph{Observed trajectories.}
In the recorded runs, SFT begins to plateau at approximately
150 steps. GRPO fluctuates more and enters its main reward
range at approximately 400 steps. OPSD begins to plateau
at approximately 250 steps, but shows greater later-stage
variability and frequent generation-limit events. CCOPD
improves steadily and begins to plateau at approximately
300 steps.

\begin{figure}[t]
    \centering
    \includegraphics[width=\columnwidth]
    {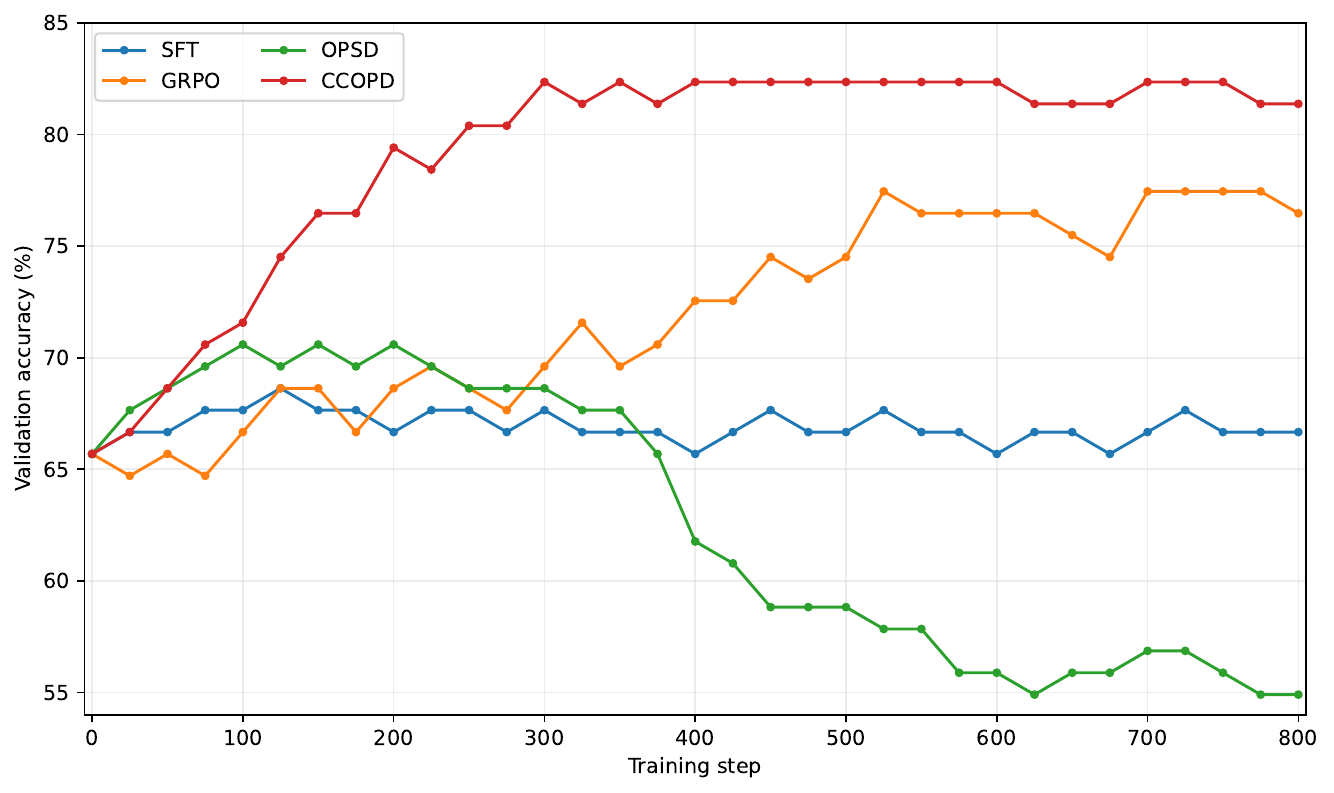}
    \caption{
    Training dynamics on the validation set.
    }
    \label{fig:training_dynamics}
\end{figure}

\subsection{Baseline Implementation and Checkpoint Selection.}
For SFT, we use the same retained RAW-SHARDED final-turn histories as CCOPD, with supervision derived from the ground-truth reference solutions rather than teacher-generated pseudo-labels.
For GRPO, we adopt an R1-style rule-based reward based on final-answer correctness, sampling eight on-policy rollouts per training problem.
For OPSD, the teacher is provided with privileged access to the ground-truth solution reasoning, while the student conditions only on the RAW-SHARDED history.
To ensure a fair comparison, all methods are trained sufficiently long to reach their main performance range, and the checkpoint achieving the highest validation (VAL) score is selected for final evaluation, rather than using a fixed training step. The validation set (VAL) consists of 102 held-out
math problems from GSM8K and GSM8K-Aug, none of
which are included in the training data. Importantly, none of the non-math evaluation
tasks or their test results are used for checkpoint
selection or hyperparameter tuning, ensuring that
the reported out-of-domain results reflect
zero-shot transfer without target-domain tuning.

% \paragraph{Diagnostic comparison.}
% RTA and DUC increase SAAR from 1.45 for the baseline to
% 1.83 and 1.64, respectively, while CCOPD reaches 2.49.
% Thus, the inference-time controls also shift attention toward
% completed task evidence relative to prior assistant
% commitments. As with the main SAAR analysis, this statistic
% is descriptive.

% \paragraph{External helper reconstruction.}
% We also evaluate a two-model inference pipeline in which a
% helper receives the completed RAW-SHARDED transcript and
% produces a consolidated request that is passed to the target
% Qwen3-8B model. Using Qwen3-4B and Qwen3-8B helpers yields
% 68.0\% and 79.6\% accuracy, respectively, compared with
% 66.0\% without a helper. CCOPD reaches 82.5\% without an
% additional helper-model call.

% \paragraph{Generation from neutralized histories.}
% We additionally evaluate the neutral replacement as a full
% generation intervention rather than only a fixed-prefix
% diagnostic. For each completed baseline RAW-SHARDED history,
% every non-final assistant reply is replaced using the
% length-matched neutral construction in Appendix~C. All user
% turns, shard order, role markers, system prompts, and turn
% boundaries remain unchanged. The model then freely generates
% the final answer, which is scored with the original math
% evaluator.

% This intervention raises baseline Math-R from 66.0\% to
% 80.6\%. The result complements the fixed-prefix diagnostic:
% removing the realized intermediate assistant content also
% improves end-to-end final generation. We treat this as a
% transcript-editing control.

\section{AI Assistant Use Disclosure}
The authors used ChatGPT as an AI assistant for limited coding support and
language editing, including drafting and debugging experimental scripts,
improving grammar and clarity, and suggesting alternative phrasing. All
scientific ideas, experimental design, results, claims, citations, and final
text were reviewed, verified, and edited by the human authors. ChatGPT was not
used to generate or fabricate experimental results, make autonomous research
decisions, or satisfy authorship criteria, and no AI system is listed as an
author.
\section{Artifact Licenses, Terms, and Intended Use}
\paragraph{{Artifact Licenses and Terms of Use}}
We use publicly available model checkpoints, benchmark datasets, and
evaluation code only for research purposes and follow the terms specified by
their original releases. Qwen3 models are released under the Apache 2.0
license, while Llama3.1 models are governed by the Meta Llama 3.1 Community
License. The evaluation and training artifacts used in this work include
GSM8K, HumanEval, LiveCodeBench, Berkeley Function-Calling Leaderboard
(BFCL), Spider, ToTTo, Summary of a Haystack (SummHay), HotpotQA, DAPO-Math-17K, AIME~2025, and the
Lost in Conversation sharded-instruction release. Their licenses or release
terms include MIT, Apache 2.0, CC BY-SA 4.0, Creative Commons Share-Alike
3.0, and model-specific community licenses, depending on the artifact.

Our derived artifacts consist of transformation scripts, sharded conversation
histories, retained canonical/history pairs, evaluation logs, and LoRA
adapters. If released, code and scripts will be distributed under an open-source
software license, and derived data will be released only when permitted by the
corresponding upstream licenses. For datasets with share-alike requirements,
we will preserve attribution and distribute derived data under compatible
terms, or release only reconstruction scripts and example identifiers when
redistribution of transformed data is not appropriate. Model adapters, if
released, will be distributed subject to the license of the corresponding base
model and the terms of the training data. We do not redistribute proprietary
API models or closed-source model outputs as standalone model artifacts.

\paragraph{Artifact Use and Intended Use}
All external artifacts are used for research purposes consistent with their
released benchmark or model-development roles. We use publicly released, post-trained instruction-following checkpoints as starting models, frozen teachers, or unmodified research baselines, and use public datasets and
evaluation suites only to study math reasoning, code generation, function
calling, text-to-SQL, table-to-text generation, summarization, and multi-turn
instruction following. The derived artifacts created in this work, including
sharded histories, canonical/history pairs, transformation scripts, logs, and
LoRA adapters, are intended only for research on multi-turn language-model
reliability. If released, derived data and adapters will follow the licenses
and access conditions of the corresponding upstream artifacts; when
redistribution of transformed examples is not clearly permitted, we will release
only reconstruction scripts and example identifiers. We do not collect private
user conversations or new human-subject data, and we do not attempt to identify
individuals or infer private attributes from benchmark examples.
\section{Artifact Documentation}
\label{app:artifact-documentation}

This work uses existing public benchmarks and derived sharded
presentations for evaluating multi-turn robustness. All language data
used in our experiments is English. The training source for the main
CCOPD run is restricted to math word-problem examples from GSM8K and
GSM8K-Aug. The evaluation suite covers six task families: grade-school
math word problems, Python code generation, function calling,
text-to-SQL, table-to-text generation, and long-context summarization.
We additionally use HotpotQA for the reciprocal
source-domain check, DAPO-Math-17K for an additional
training-source experiment, and AIME~2025 as a harder
mathematical reasoning stress test.

\paragraph{Domains and task coverage.}
The math domain consists of grade-school arithmetic word problems.
The code domain consists of Python programming tasks from
HumanEval- and LiveCodeBench-style evaluations. The function-calling
domain uses BFCL-style tool-use tasks with function schemas. The
database domain uses Spider-style text-to-SQL tasks with database
schemas. The data-to-text domain uses ToTTo-style table description
tasks. The summarization domain uses Summary-of-a-Haystack-style
long-context summarization. These domains are used to test whether a
model trained only on sharded math conversations transfers to
non-math RAW-SHARDED histories.

\paragraph{Constructed sharded histories.}
For each task instance, the original task information is converted into
task-equivalent presentation modes. \textsc{Full} presents the complete
task in one prompt, \textsc{Concat} concatenates all user-provided
shards, and \textsc{Raw-Sharded} reveals the same user evidence
incrementally through a multi-turn interaction. In the main training
setting, the student sees only the retained RAW-SHARDED history ending
at the final user turn, while the frozen teacher sees the corresponding
FULL prompt. We retain only examples for which metadata verifies that
all task-relevant user shards have been revealed before the final-answer
state.

\paragraph{Linguistic and interaction phenomena.}
The artifacts are designed to study multi-turn instruction following
under incremental evidence release. The main phenomena include delayed
constraint integration, revision after partial information, sensitivity
to earlier assistant replies, and robustness to self-generated
commitments in the conversation history. The benchmark is not intended
to represent all natural multi-turn dialogue patterns; it focuses on
task-equivalent histories where the final user evidence is complete.

\paragraph{Demographic information.}
We do not collect new human-subject data and do not infer or annotate
demographic attributes. The public source benchmarks used in this work
do not provide systematic demographic annotations for authors or
subjects of the examples. Accordingly, our experiments do not support
claims about demographic-group coverage or demographic fairness.

\paragraph{Evaluation artifacts.}
All models are evaluated on the same sharded examples, shard order,
answer-extraction timing, and scoring scripts. Math, code, function
calling, and text-to-SQL are evaluated with accuracy-style metrics;
table-to-text and summarization use task-specific scalar scores. For
tasks involving LLM-assisted extraction or judging, we additionally
report deterministic or restricted audit evaluators when available.

\end{document}